\documentclass[10pt]{article}

\usepackage[utf8]{inputenc}
\usepackage[T1]{fontenc}

\usepackage{epsf}
\usepackage{amsmath}

\allowdisplaybreaks

\usepackage[showframe=false]{geometry}
\usepackage{changepage}

\usepackage{epsfig}
\usepackage{amssymb}

\usepackage{amsthm}
\usepackage{setspace}
\usepackage{cite}
\usepackage{mcite}

\usepackage{algorithmic}  
\usepackage{algorithm}

\usepackage{shadow}
\usepackage{fancybox}
\usepackage{fancyhdr}

\usepackage{color}
\usepackage[usenames,dvipsnames,svgnames,table]{xcolor}
\newcommand{\bl}[1]{\textcolor{blue}{#1}}
\newcommand{\red}[1]{\textcolor{red}{#1}}

\definecolor{mypurple}{rgb}{.4,.0,.5}

\usepackage[hyphens]{url}

\usepackage[colorlinks=true,
            linkcolor=black,
            urlcolor=blue,
            citecolor=purple]{hyperref}

\usepackage{breakurl}

\def\y{{\bf y}}

\def\x{{\bf x}}

\def\x{{\mathbf x}}

\def\u{{\bf u}}

\def\x{{\bf x}}
\def\y{{\bf y}}

\def\q{{\bf q}}
\def\m{{\bf m}}

\def\b{{\bf b}}
\def\c{{\bf c}}

\def\h{{\bf h}}

\def\cH{{\mathcal H}}

\def\be{\begin{equation}}
\def\ee{\end{equation}}
\def\ba{\left[\begin{array}}
\def\ea{\end{array}\right]}

\def\u{{\bf u}}

\def\x{{\bf x}}
\def\y{{\bf y}}

\def\q{{\bf q}}

\def\b{{\bf b}}
\def\c{{\bf c}}

\def\p{{\bf p}}

\def\1{{\bf 1}}

\def\0{{\bf 0}}

\def\erfc{\mbox{erfc}}

\def\mR{{\mathbb R}}
\def\mN{{\mathbb N}}
\def\mE{{\mathbb E}}
\def\mS{{\mathbb S}}
\def\mB{{\mathbb B}}
\def\mP{{\mathbb P}}

\def\lp{\left (}
\def\rp{\right )}

\def\y{{\bf y}}

\def\x{{\bf x}}

\def\x{{\mathbf x}}

\def\u{{\bf u}}

\def\x{{\bf x}}
\def\y{{\bf y}}

\def\q{{\bf q}}

\def\b{{\bf b}}
\def\c{{\bf c}}

\def\h{{\bf h}}

\def\cH{{\cal H}}

\def\be{\begin{equation}}
\def\ee{\end{equation}}
\def\ba{\left[\begin{array}}
\def\ea{\end{array}\right]}

\def\u{{\bf u}}

\def\x{{\bf x}}
\def\y{{\bf y}}

\def\q{{\bf q}}

\def\b{{\bf b}}
\def\c{{\bf c}}

\def\p{{\bf p}}

\def\({\left (}
\def\){\right )}

\def\1{{\bf 1}}
\def\m{{\bf m}}
\def\q{{\bf q}}

\def\0{{\bf 0}}

\def\cX{{\mathcal X}}
\def\cY{{\mathcal Y}}

\usepackage{xcolor}
\usepackage{color}

\definecolor{darkgreen}{rgb}{0, 0.4,0}

\definecolor{purplebrown}{rgb}{0.5,0.1,0.6}

\definecolor{ultclupcol}{rgb}{0.1,0.5,0.5}

\definecolor{mytrycolor}{rgb}{0.5,0.7,0.2}

\definecolor{ultclupcola}{rgb}{.5,0,.5}

\definecolor{shadebrown}{rgb}{0.1,0.1,0.9}
\definecolor{lightblue}{rgb}{0.2,0,1}

\usepackage{fancybox}
\usepackage{graphicx}
\usepackage{epstopdf}
\usepackage{epsfig}
\usepackage{wrapfig}
\usepackage{subfigure}

\usepackage{xcolor}
\usepackage{tcolorbox}

\newtcbox{\xmybox}{on line,
arc=7pt,
before upper={\rule[-3pt]{0pt}{10pt}},boxrule=0pt,
boxsep=0pt,left=6pt,right=6pt,top=0pt,bottom=0pt,enhanced, coltext=blue, colback=white!10!yellow}

\newtcbox{\xmyboxa}{on line,
arc=7pt,
before upper={\rule[-3pt]{0pt}{10pt}},boxrule=0pt,
boxsep=0pt,left=6pt,right=6pt,top=0pt,bottom=0pt,enhanced, colback=white!10!yellow}

\newtcbox{\xmyboxb}{on line,
arc=7pt,
before upper={\rule[-3pt]{0pt}{10pt}},boxrule=1pt,colframe=darkgreen!100!blue,
boxsep=0pt,left=6pt,right=6pt,top=0pt,bottom=0pt,enhanced, colback=white!10!yellow}

\newtcbox{\xmyboxc}{on line,
arc=7pt,
before upper={\rule[-3pt]{0pt}{10pt}},boxrule=.7pt,colframe=blue!100!blue,
boxsep=0pt,left=6pt,right=6pt,top=0pt,bottom=0pt,enhanced, coltext=blue, colback=white!10!yellow}

\newtcbox{\xmytboxa}{on line,
arc=7pt,
before upper={\rule[-3pt]{0pt}{10pt}},boxrule=.0pt,colframe=pink!50!yellow,
boxsep=0pt,left=6pt,right=6pt,top=0pt,bottom=0pt,enhanced, coltext=white, colback=blue!40!red}

\newtcbox{\xmytboxb}{on line,
arc=7pt,
before upper={\rule[-3pt]{0pt}{10pt}},boxrule=.0pt,colframe=pink!50!yellow,
boxsep=0pt,left=6pt,right=6pt,top=0pt,bottom=0pt,enhanced, coltext=white, colback=white!40!green}

\makeatletter
\newcommand\subsubsubsection{\@startsection{paragraph}{4}{\z@}{-2.5ex\@plus -1ex \@minus -.25ex}{1.25ex \@plus .25ex}{\normalfont\normalsize\bfseries}}
\newcommand\subsubsubsubsection{\@startsection{subparagraph}{5}{\z@}{-2.5ex\@plus -1ex \@minus -.25ex}{1.25ex \@plus .25ex}{\normalfont\normalsize\bfseries}}
\makeatother

\newtheorem{theorem}{Theorem}

\begin{document}

\begin{singlespace}

\title {Binary perceptron computational gap -- a parametric fl RDT view 
}
\author{
\textsc{Mihailo Stojnic
\footnote{e-mail: {\tt flatoyer@gmail.com}} }}
\date{}
\maketitle

\centerline{{\bf Abstract}} \vspace*{0.1in}

Recent  studies suggest that asymmetric binary perceptron (ABP) likely exhibits the so-called statistical-computational gap characterized with the appearance of two phase transitioning constraint density thresholds: \textbf{\emph{(i)}} the \emph{satisfiability threshold} $\alpha_c$, below/above which ABP succeeds/fails to operate as a storage memory; and \textbf{\emph{(ii)}} \emph{algorithmic threshold} $\alpha_a$, below/above which one can/cannot efficiently determine ABP's weight so that it operates as a storage memory.

We consider a particular parametric utilization of  \emph{fully lifted random duality theory} (fl RDT) \cite{Stojnicflrdt23} and study its potential ABP's algorithmic implications. A remarkable structural parametric change is uncovered as one progresses through fl RDT lifting levels. On the first two levels,  the so-called $\c$ sequence -- a key parametric fl RDT component -- is of the (natural) decreasing type. A change of such phenomenology on higher levels is then connected to the $\alpha_c$ -- $\alpha_a$ threshold change. Namely, on the second level concrete numerical values give for the critical constraint density  $\alpha=\alpha_c\approx 0.8331$. While progressing through higher levels decreases this estimate, already on the fifth level we observe a satisfactory level of convergence and obtain $\alpha\approx 0.7764$. This allows to draw two striking parallels: \textbf{\emph{(i)}} the obtained constraint density estimate is in a remarkable agrement with range $\alpha\in (0.77,0.78)$ of clustering defragmentation (believed to be responsible for failure of locally improving algorithms) \cite{Bald16,Stojnicabple25}; and \textbf{\emph{(ii)}} the observed change of $\c$ sequence phenomenology closely matches the one of the negative Hopfield model for which the existence of efficient algorithms that closely approach similar type of threshold has been demonstrated recently \cite{Stojniccluphop25}.

\vspace*{0.25in} \noindent {\bf Index Terms: Binary perceptrons; Fully lifted random duality theory; Algorithmic threshold}.

\end{singlespace}

\section{Introduction}
\label{sec:back}

The last two decades brought an unprecedented progress in the development of AI.  To a large degree it is enabled by algorithmic and theoretical machine learning (ML) and neural networks (NN) studies of several preceding decades. Either as irreplaceable integral parts of more complex NN structures or as individual prototype models closely resembling behavior of more generic ML architectures, perceptrons have been among the most often studied AI topics. Two classical perceptrons types, \emph{spherical} and \emph{binary}  (SP and BP), distinguished themselves in that regard. It is likely the combination of their sufficient simplicity and generality that paved the way towards such developments. Simplicity allowed analytical tractability whereas generality enabled adequate description of artificial reasoning. Along these lines, it is a no stretch to say that  \cite{Wendel62,Winder,Cover65}'s determination of SP's capacity (data density $\alpha_c$ below which storage/classifying is possible) is among key breakthroughs that shaped analytical AI considerations for the following several decades. Its simplicity, applicability, and timeliness crucially contributed towards raising awareness regarding importance of strong analytical foundation in the overall AI success. Concurrently, it fruitfully interconnected many distant social, engineering, and scientific fields ranging from logic, psychology, and cognitive thinking to information theory, optimization, algorithms, and statistical physics. In a large part due to these early efforts, imagining modern AI without a strong scientifically diverse theoretical support is nowadays almost impossible.

\subsection{Analytical difficulties}
\label{secanaldiff}

Initial SP considerations of  \cite{Wendel62,Winder,Cover65} were followed by  a large body of highly influential work
\cite{Gar88,GarDer88,SchTir02,SchTir03,Tal05,Talbook11a,Talbook11b,StojnicGardGen13,StojnicGardSphNeg13} which deepened understanding of perceptrons and widened the horizons allowing shifting the focus to more complex ML structures and features beyond capacities. A subclass of SPs, the so-called  \emph{positive} spherical perceptrons (PSP),   became particularly well  understood and the presence of strong deterministic duality/convexity as their underlying features enabled successful analytical studies \cite{SchTir02,SchTir03,StojnicGardGen13} that rigorously proved replica predictions \cite{Gar88,GarDer88} and significantly superseded \cite{Wendel62,Winder,Cover65}. In parallel, absence  of these features was perceived as  analytically often unsurpassable and a direct obstacle for a repetition of \cite{SchTir02,SchTir03,StojnicGardGen13}'s  success. The negative spherical perceptron (NSP) counterpart served as the prime example of a fairly simple model where a rather minimal deviation from the PSP's positive threshold removes the underlying convexity and makes analytical considerations substantially harder \cite{StojnicGardSphNeg13,FPSUZ17,FraHwaUrb19,FraPar16,FraSclUrb19,FraSclUrb20,AlaSel20,BMPZ23}. Moreover, compared to classical SP characterizations \cite{SchTir02,SchTir03,StojnicGardGen13,Wendel62,Winder,Cover65}, much more sophisticated  approaches turned out to be needed   \cite{Stojnicsflgscompyx23,Stojnicnflgscompyx23,Stojnicflrdt23,Stojnicnegsphflrdt23}.

Similar absence of the strong deterministic duality is generically featured in BPs as well. Consequently, in \emph{asymmetric} BP (ABP) simple replica symmetric (RS) predictions do not hold \cite{Gar88,GarDer88,StojnicDiscPercp13}. Instead, to achieve accurate capacity characterizations, more involved,  replica symmetry breaking (RSB) ones from \cite{KraMez89} are needed \cite{DingSun19,NakSun23,BoltNakSunXu22,Huang24,Stojnicbinperflrdt23}. It is interesting to observe that another BP class, the so-called \emph{symmetric} BP (SBP), exhibits a so to say mixed behavior. RS predictions again do not hold, but a favorable combinatorial nature of the underlying problems allows for simple analytical characterizations \cite{AbbLiSly21b,PerkXu21,AbbLiSly21a,AubPerZde19,GamKizPerXu22}, which in return positions SBPs somewhere in between the easy PSPs and the hard NSPs and ABPs.

\subsection{Algorithmic difficulties}
\label{sec:examples}

The above discussion relates to perceptrons' analytical/theoretical limits which determine the ultimate storage/classifying power. To what degree such a power can be fully utilized is a different question. The answer actually depends on one's ability to \emph{efficiently} determine the perceptron's weights so that the constraint density $\alpha_c$ can be accommodated.  For BPs this is always possible but it might take a large amount of computational time and as such it may be deemed as practically inefficient or even infeasible. Thinking along these lines brings up the importance of the algorithmic component in studying BPs (and NNs in general). In fact, one quickly observes that it is not necessary that perceptrons theoretical capacity limits can be achieved in a computationally efficient way. This on the other hand allows to introduce $\alpha_a$ as the critical constraint density for which BP's weights can be determined efficiently. In general $\alpha_a\leq \alpha_c$ and if $\alpha_a< \alpha_c$ one then has a \emph{computational} gap (C-gap). The size of the gap typically serves as a measure of utilization of the overall  BP's predicated power.

Determining both $\alpha_a$ and $\alpha_c$ is of extraordinary importance. However, in general it is a challenging task. For example, classical complexity theory positions ABP's algorithmic solving as an NP problem \cite{Ama91} which likely implies $\alpha_a< \alpha_c$, i.e., the existence of a gap. However, by definition the NP-ness is a \emph{worst case} concept and as such it rarely properly addresses \emph{typical} algorithmic solvability. On the other hand, \emph{statistical} scenarios are viewed as more reflective of typical behavior and in such contexts \emph{statistical-computational} gap is the equivalent utilization measure of ABP's power. From a practical point of view, things are additionally complicated by the fact that many efficient algorithms that perform well in a large part of $\alpha< \alpha_c$ range actually exist \cite{BrZech06,BaldassiBBZ07,Hubara16,KimRoc98}. More concretely, while ABP's theoretical capacity is $\alpha_c\approx 0.8331$, the best available algorithms \cite{BaldassiBBZ07,Bald15,BMPZ23} suggest  $\alpha_{a}\approx 0.75 - 0.77$. Clearly, as stated above, determining precise value of  $\alpha_{a}$,   below which efficient algorithms exist is an extraordinary challenge. Uncovering structural phenomenology behind existence of such algorithms is likely as challenging. As the current state of the art has $\alpha_{a}<\alpha_c$, it is widely believed that the gap indeed exists (examples of other optimization problems (including planted ones) with similar algorithmic implications can be found in, e.g., \cite{MMZ05,GamarSud14,GamarSud17,GamarSud17a,AchlioptasR06,AchlioptasCR11,GamMZ22,BarbierKMMZ18,KMSSZ12a}).

\subsection{Relevant prior work}
\label{sec:examples}

Demystification of computational gaps has been the subject of extensive studies over the last two decades. Despite a strong progress on particular problems a generic resolution still seems far away. We here focus on two approaches related to clustering of typical/atypical solutions that gained a lot of interest in recent years: (i) \emph{Overlap gap property} (OGP) (see, e.g., \cite{Gamar21,GamarSud14,GamarSud17,GamarSud17a,AchlioptasCR11,HMMZ08,MMZ05}); and (ii) \emph{Local entropy} (LE) (see, e.g., \cite{Bald15,Bald16,Bald20}).

The OGP approach \cite{Gamar21,GamarSud14,GamarSud17,GamarSud17a,AchlioptasCR11,HMMZ08,MMZ05} connects algorithmic efficiency and gaps in the spectrum of attainable solutions (or near-solutions) overlaps. It basically postulates that if the gaps are absent then efficient algorithms exist. This then means that $\alpha_a$ can be viewed as maximal $\alpha$ such that OGP is absent. For structurally similar (but analytically easier) SBP alternative the OGP's presence extends well below $\alpha_c$ \cite{GamKizPerXu22,Bald20} (corresponding discrepancy minimization results can be found in, e.g., \cite{GamKizPerXu23}). Provided predicated OGP algorithmic relevance, this strongly suggests that C-gap indeed exists. The shortest path counterexample from \cite{LiSch24} disproves OGP generic hardness implications (earlier disproving examples were of a simple algebraic nature and as such viewed as exceptions). However,  \cite{LiSch24} does not disprove OGP's relevance for different problems or specific algorithms. For example, in famous $2$-spin Ising Sherrington-Kirkpatrick (SK) model \cite{SheKir72}, (widely believed) absence of OGP directly implies polynomial solvability \cite{Montanari19} (for corresponding $p$-spin results see, e.g.,  \cite{AlaouiMS22,AlaouiMS21}; for earlier spherical SK models  related considerations see, e.g.,  \cite{Subag17,Subag17a,Subag21,Subag24}; for analogous NSP discussions see, e.g., \cite{AlaSel20,AMZ24}; and for relevance of more sophisticated OGPs see, e.g., \cite{Kiz23,HuangS22}). While at present it remains undetermined what role OGP ultimately plays in generic algorithmic hardness, its presence disallows efficient implementations for many particular types of algorithms  \cite{GamKizPerXu22}. Moreover,  for many well known problems \cite{RahVir17,GamarSud14,GamarSud17,GamAW24,Wein22},  practical algorithms are known to exist in $\alpha$ ranges where OGP is absent.

Proceeding in a direction seemingly different from OGP, \cite{Huang13,Huang14} connects  algorithmic hardness relevance of clustering organization to entropies of \emph{typical} solutions. A completely \emph{frozen} typical solutions isolation is predicated (and proven for SBP in \cite{PerkXu21,AbbLiSly21a,AbbLiSly21b}).  \cite{Bald15,Bald16,Bald20}  consider a similar concept but with a stronger refinement. Namely, they study \emph{local entropy} (LE) of \emph{atypical} well-connected clusters. Roughly speaking, the idea is that even when predominant typical solutions are disconnected  (and presumably unreachable via local searches) \cite{Huang13,Huang14,PerkXu21,AbbLiSly21b}, rare (atypical) well-connected clusters may still exist. It is then predicated that  such rare clusters are precisely those that efficient algorithms somehow magically find (for an SBP's sampling type of justification that goes along these lines, see, e.g., \cite{ElAlGam24}). Provided correctness of such a pictorial portrayal, the C-gap existence is then likely in direct correlation with properties of rare clusters. \cite{Bald15,Bald16,Bald20} further speculate that LE features (monotonicity, breakdown, or even negativity) might be a key reflection of rare clusters' structures and the associated algorithmic hardness relevance. Such a phenomenology is supported by results of \cite{AbbLiSly21a} where clusters of maximal diameter are shown to exist in SBPs with sufficiently small $\alpha$. \cite{AbbLiSly21a} further showed that similar clusters  (albeit of linear diameter) actually exist for any $\alpha<\alpha_c$ (with additional technical assumptions, \cite{AbbLiSly21a}'s SBP results translate to ABP as well). To reconnect with the OGP,  the small $\alpha$ SBP LE results are shown in \cite{BarbAKZ23} to closely match (scaling-wise) the \cite{GamKizPerXu22}'s OGP predictions  (modulo a log term, they also match the algorithmic performance achieved in  \cite{BanSpen20}). Such an OGP -- LE correspondence is certainly nice and convincing. It likely presents steps in the right direction towards establishing a definite answer regarding the role of these phenomena in the appearance of C-gaps. While the overall demystification still remains a grand challenge, it is useful to emphasize  that, irrespective of ultimate C-gap relevance, the above discussed OGP and LE phenomena provide deep insights into the intrinsic organization of random structures and understanding them is most definitely of independent interest as well.

\subsection{Our contributions}
\label{sec:cont}

We here focus on an entirely different approach and study potential ABP's algorithmic implications via a particular parametric utilization of a powerful mathematical machinery called \emph{fully lifted random duality theory} (fl RDT) \cite{Stojnicflrdt23}. We uncover that within the ABP context the fl RDT exhibits a remarkable structural parametric change as one progresses through lifting levels. On the first two levels,  the so-called $\c$ sequence -- one of the key parametric fl RDT components -- is of the natural (physical) decreasing type. Moving to higher levels this phenomenology abruptly changes and a perfect $\c$ ordering is not present any longer. We connect such a change to the change from satisfiability to algorithmic threshold. Through concrete numerical evaluations we find that the constraint density $\alpha$ on the second lifting level precisely matches the satisfiability threshold $\alpha_c\approx 0.8331$. As one continues progress through higher lifting levels this estimate decreases. However, already on the fifth lifting level we obtain $\alpha\approx 0.7764$ and observe a clear converging tendency with the difference between successive lifting levels estimates on the order of $\sim 0.001$.

These developments are then connected to the above mentioned studies of ABP's atypical solutions clusterings. In particular, one observes that they fairly closely match the LE results of \cite{Bald16,Stojnicabple25} which predict the clustering defragmentation (likely responsible for failure of locally improving algorithms)  for $\alpha\in (0.77,0.78)$. A further parallel is then drawn with recent algorithmic studies of the negative Hopfield models (Hop-) \cite{Stojniccluphop25}. Namely, a structural parametric similarity regarding the above mentioned $\c$ sequence is observed in two models.
Since \cite{Stojniccluphop25} has also presented efficient algorithms that already for fairly small dimensions on the order of a few thousands closely approach higher lifting levels theoretical  Hop- predictions, it is reasonable to believe that equally successful and conceptually similar ones can be designed for ABP as well. Moreover, we also believe that the presented parametric algorithmic phenomenology extends beyond the ABP. As such it might likely be a consequence of a more generic principle and therefore applicable in determination of statistical-computational gaps in other random optimization problems. Possible examples include feasibility problems such as SBPs  \cite{AubPerZde19,AbbLiSly21a,AbbLiSly21b,Bald20,GamKizPerXu22,PerkXu21,ElAlGam24,SahSaw23,Barb24,djalt22,BarbAKZ23},
NSPs \cite{StojnicGardGen13,StojnicGardSphNeg13,GarDer88,Gar88,Schlafli,Cover65,Winder,Winder61,Wendel62,Cameron60,Joseph60,BalVen87,Ven86,SchTir02,SchTir03,AMZ24,BMPZ23}, as well as  optimal objective seeking standard optimizations ones such as the above mentioned Hop- and  discrepancy minimization \cite{Stojniccluphop25,KaKLO86,Spen85,LovMek15,GamKizPerXu23,Roth17,AlwLiuSaw21}.

\section{Mathematical ABP preliminaries}
 \label{sec:bprfps}

As observed in \cite{StojnicDiscPercp13,StojnicGardGen13,Stojnicbinperflrdt23,Stojnicnegsphflrdt23,GarDer88,Gar88,StojnicGardSphNeg13},
for two positive integers $m$ and $n$, ABPs are a particular instance of the following general class of  \emph{feasibility} problems with linear inequalities
\begin{eqnarray}
\hspace{-1.5in}\mbox{$\mathbf{\mathcal F}(G,\b,\cX,\alpha)$:} \hspace{1in}\mbox{find} & & \x\nonumber \\
\mbox{subject to}
& & G\x\geq \b \nonumber \\
& & \x\in\cX. \label{eq:ex1}
\end{eqnarray}
$G\in\mR^{n\times n}$, $\b\in\mR^{m\times 1}$, $\cX\in\mR^n$, and $\alpha=\frac{m}{n}$ represent the given infrastructure of the problem and determine the type of perceptron (throughout the paper,  \emph{linear/proportional} regime is considered with $m$ and $n$ such that $\alpha= \lim_{n\rightarrow\infty} \frac{m}{n}$ remains constant). Several prominent perceptron types mentioned earlier are obtained in the following way. For example, taking $\cX=\{\x | \| \x \|_2=1\} \triangleq \mS^m$ with $\b\geq 0$ gives PSP \cite{StojnicGardGen13,GarDer88,Gar88,Schlafli,Cover65,Winder,Winder61,Wendel62,Cameron60,Joseph60,BalVen87,Ven86,SchTir02,SchTir03}. The same specialization but with $\b< 0$ gives NSP  \cite{AMZ24,BMPZ23,FPSUZ17,Talbook11a,FraHwaUrb19,FraPar16,FraSclUrb19,FraSclUrb20,AlaSel20,StojnicGardSphNeg13,Stojnicnegsphflrdt23}. On the other hand,
 taking $\cX=\{-\frac{1}{\sqrt{n}},\frac{1}{\sqrt{n}} \}^n \triangleq \mB^n$ gives ABP \cite{Talbook11a,StojnicGardGen13,GarDer88,Gar88,StojnicDiscPercp13,KraMez89,GutSte90,KimRoc98,NakSun23,BoltNakSunXu22,PerkXu21,CXu21,DingSun19,Huang24,Stojnicbinperflrdt23,LiSZ24}. Additional change in structure of linear constraints, $|G\x |\leq \b$, gives SBP \cite{AubPerZde19,AbbLiSly21a,AbbLiSly21b,Bald20,GamKizPerXu22,PerkXu21,ElAlGam24,SahSaw23,Barb24,djalt22,BarbAKZ23} (or closely related discrepancy minimization problems \cite{KaKLO86,Spen85,LovMek15,GamKizPerXu23,Roth17,AlwLiuSaw21}). The above are variants with variable thresholds. Taking further  $\b=\kappa\1$ (with $\kappa\in\mR$ and $\1$ being the vector of all ones), one obtains corresponding \emph{fixed} threshold counterparts. Clearly, generic and random $G$ give deterministic and statistical perceptrons, respectively. We focus on the classical Gaussian perceptrons where the elements of $G$ are independent standard normals.

After setting
\begin{eqnarray}
\xi_{ABP}
& =  &
 \min_{\x\in \mB^n} \max_{\y\in\mS_+^m}  \lp -\y^TG\x + \kappa \y^T\1 \rp,
 \label{eq:ex3}
\end{eqnarray}
with $\mS_+^m$ being the  positive orthant part of the $m$-dimensional unit sphere  (i.e., $\mS_+^m=\{\y|\|\y\|_2=1,\y\geq 0\}$) and $\mB^n$ being the vertices of the $n$-dimensional unit cube,
 \cite{StojnicGardGen13,StojnicDiscPercp13,Stojnicbinperflrdt23} recognized  ABP's storage/classifying statistical  capacity as
 \begin{eqnarray}
\alpha & = &    \lim_{n\rightarrow \infty} \frac{m}{n}  \nonumber \\
\alpha_c(\kappa) & \triangleq & \max \{\alpha |\hspace{.08in}  \lim_{n\rightarrow\infty}\mP_G\lp  \xi_{ABP}>0\rp\longrightarrow 1\} \nonumber \\
& = & \max \{\alpha |\hspace{.08in}  \lim_{n\rightarrow\infty}\mP_G\lp{\mathcal F}(G,\b,\cX,\alpha) \hspace{.07in}\mbox{is feasible} \rp\longrightarrow 1\}.
  \label{eq:ex4}
\end{eqnarray}
Removal of $\mP_G$ gives the corresponding deterministic capacity variant (throughout the paper we adopt the convention that if the subscript next to $\mP$ and/or $\mE$ is present, it denotes the source of randomness with respect to which one does the statistical evaluation). Clearly, in (\ref{eq:ex4}) the randomness is taken with respect to $G$. Also, given statistical context, to facilitate writing, we throughout the paper often skip repeating that statements hold with high probability.

\subsection{Capacity relevance}
 \label{sec:caprole}

The capacity is the critical constraint density (ratio of the number of data constraints, $m$, and the ambient dimension, $n$) for which the ABP can properly operate. In statistical contexts of interest here, if also reflects the phase-transitioning nature of the underlying randomness. In particular, one has for $\alpha>\alpha_c$ that (\ref{eq:ex1}) is infeasible or in the so-called UNSAT phase.  Analogously, when $\alpha<\alpha_c$, (\ref{eq:ex1}) is feasible and in the SAT phase.  Exactly at $\alpha_c$ one has a transition between these phases characterized with an exponentially large number of solutions  shrinking to an empty set \cite{KraMez89,NakSun23,BoltNakSunXu22,DingSun19,Huang24,Stojnicbinperflrdt23}. This represents a theoretical limit and effectively characterizes ultimate ABP's storage/classifying power. To be able to utilize all of that power one needs to determine the perceptron weights. While this is always possible, it might be computationally inefficient. In fact, there are no known efficient algorithms that provably solve ABP for $\alpha\approx \alpha_c$. Somewhat paradoxically, they do exist for $\alpha\sim 0.75-0.77 <   0.8331 \approx \alpha_c$. Precisely characterizing this discrepancy amounts to determining $\alpha_a$ as a critical constraint density below which fast algorithms exist. Determining $\alpha_a$ is further equivalent to determining the existence and size of the statistical-computational gap. It is an extraordinary challenge that we attack in this paper via a parametric fl RDT approach.

\subsection{Random feasibility -- free energy connection}
\label{secrfpsfe}

The above random feasibility problems can be incorporated into equivalent free energy contexts. We consider the so-called Hamiltonian
\begin{equation}
\cH(G)= \y^TG\x,\label{eq:ham1}
\end{equation}
and the corresponding (virtual) partition function
\begin{equation}
Z(\beta,G)=\sum_{\x\in\cX} \lp \sum_{\y\in\cY}e^{\beta \lp  \cH(G) + f(\y)  \rp }\rp^{-1},  \label{eq:partfun}
\end{equation}
with $\cX$ and $\cY$ being general sets (later on, we specialize to particular sets of our interest). We can then associate the following thermodynamic limit (average ``\emph{reciprocal}'') free energy
\begin{eqnarray}
f_{sq}(\beta) & = & \lim_{n\rightarrow\infty}\frac{\mE_G\log{(Z(\beta,G)})}{\beta \sqrt{n}}
=\lim_{n\rightarrow\infty} \frac{\mE_G\log\lp \sum_{\x\in\cX} \lp \sum_{\y\in\cY}e^{\beta \lp   \cH(G)  + f(\y) \rp   }\rp^{-1}\rp}{\beta \sqrt{n}} \nonumber \\
& = &\lim_{n\rightarrow\infty} \frac{\mE_G\log\lp \sum_{\x\in\cX} \lp \sum_{\y\in\cY}e^{\beta \lp   \y^TG\x   + f(\y) \rp  )}\rp^{-1}\rp}{\beta \sqrt{n}}.\label{eq:logpartfunsqrt}
\end{eqnarray}
Specializing to $\beta\rightarrow \infty$ we obtain the ground state analogue
\begin{eqnarray}
f_{sq}(\infty)   \triangleq    \lim_{\beta\rightarrow\infty}f_{sq}(\beta) & = &
\lim_{\beta,n\rightarrow\infty}\frac{\mE_G\log{(Z(\beta,G)})}{\beta \sqrt{n}}
=
 \lim_{n\rightarrow\infty}\frac{\mE_G \max_{\x\in\cX}  -  \max_{\y\in\cY} \lp \y^TG\x  + f(\y)\rp   }{\sqrt{n}} \nonumber \\
& = & - \lim_{n\rightarrow\infty}\frac{\mE_G \min_{\x\in\cX}  \max_{\y\in\cY} \lp \y^TG\x  + f(\y)  \rp    }{\sqrt{n}}.
  \label{eq:limlogpartfunsqrta0}
\end{eqnarray}
Imposing the randomness via $G$ comprised of  independent standard normals, we  have
\begin{equation}
-f_{sq}(\infty)
 =   \lim_{n\rightarrow\infty}\frac{\mE_G \min_{\x\in\cX}  \max_{\y\in\cY}  \lp  \y^TG\x  + f(\y) \rp   }{\sqrt{n}}  = \lim_{n\rightarrow\infty}\frac{\mE_G \min_{\x\in\cX}  \max_{\y\in\cY} \lp -\y^TG\x   + f(\y) \rp }{\sqrt{n}}.
  \label{eq:limlogpartfunsqrt}
\end{equation}
One then notes a direct connection between $f_{sq}(\infty)$ and $\xi_{feas}(0,\cX)$ from (\ref{eq:ex3}). This connection immediately implies that characterization of  $f_{sq}(\infty)$ is sufficient to characterize (\ref{eq:ex3}). However, since working directly with $f_{sq}(\infty)$ is usually hard, we focus on studying $f_{sq}(\beta)$ (where $\beta$ is general) and then eventually specialize to the above mentioned ground state energy (GSE) regime, $\beta\rightarrow\infty$. Keeping in mind that this specialization is eventually awaiting, we may on occasion neglect some terms of no GSE relevance.

\subsection{Technical preliminaries for fitting ABP into sfl RDT framework}
\label{sec:randlincons}

To be able to fit ABP into sfl RDT framework several technical preliminaries are needed. We first note that the above free energy given in (\ref{eq:logpartfunsqrt}),
\begin{eqnarray}
f_{sq}(\beta) & = &\lim_{n\rightarrow\infty} \frac{\mE_G\log\lp \sum_{\x\in\cX} \lp \sum_{\y\in\cY}e^{\beta \lp \y^TG\x +f(\y) \rp }\rp^{-1}\rp}{\beta \sqrt{n}},\label{eq:hmsfl1}
\end{eqnarray}
can be viewed as a function of bilinearly indexed random process (blirp) $\y^TG\x$. To connect the blirp results from \cite{Stojnicsflgscompyx23,Stojnicnflgscompyx23,Stojnicflrdt23} and  $f_{sq}$, we closely follow \cite{Stojnicbinperflrdt23}.  We take $r\in\mN$ and consider  vectors $\p=[\p_0,\p_1,\dots,\p_{r+1}]$, $\q=[\q_0,\q_1,\dots,\q_{r+1}]$, and $\c=[\c_0,\c_1,\dots,\c_{r+1}]$ such that
 \begin{eqnarray}\label{eq:hmsfl2}
1=\p_0\geq \p_1\geq \p_2\geq \dots \geq \p_r\geq \p_{r+1} & = & 0 \nonumber \\
1=\q_0\geq \q_1\geq \q_2\geq \dots \geq \q_r\geq \q_{r+1} & = &  0,
 \end{eqnarray}
$\c_0=1$, $\c_{r+1}=0$. Let two given sets, $\cX$ and $\cY$, be such that $\cX\subseteq \mS^n$ and $\cY\subseteq \mS^m$ (with $\mS^p$ being the unit sphere in $\mR^p$). Moreover, for $k\in\{1,2,\dots,r+1\}$ let ${\mathcal U}_k\triangleq [u^{(4,k)},\u^{(2,k)},\h^{(k)}]$  with components of  $u^{(4,k)}\in\mR$, $\u^{(2,k)}\in\mR^m$, and $\h^{(k)}\in\mR^n$ being independent standard normals. We then take real $s$ such that $s^2=1$ and for a given function $f_S(\cdot):\mR^n\rightarrow R$ set
  \begin{eqnarray}\label{eq:fl4}
\psi_{S,\infty}(f_{S},\cX,\cY,\p,\q,\c,s)  =
 \mE_{G,{\mathcal U}_{r+1}} \frac{1}{n\c_r} \log
\lp \mE_{{\mathcal U}_{r}} \lp \dots \lp \mE_{{\mathcal U}_3}\lp\lp\mE_{{\mathcal U}_2} \lp \lp Z_{S,\infty}\rp^{\c_2}\rp\rp^{\frac{\c_3}{\c_2}}\rp\rp^{\frac{\c_4}{\c_3}} \dots \rp^{\frac{\c_{r}}{\c_{r-1}}}\rp, \nonumber \\
 \end{eqnarray}
where
\begin{eqnarray}\label{eq:fl5}
Z_{S,\infty} & \triangleq & e^{D_{0,S,\infty}} \nonumber \\
 D_{0,S,\infty} & \triangleq  & \max_{\x\in\cX } s \max_{\y\in\cY }
 \lp \sqrt{n} f_{S}
+\sqrt{n}      \lp\sum_{k=2}^{r+1}c_k\h^{(k)}\rp^T\x
+ \sqrt{n}  \y^T\lp\sum_{k=2}^{r+1}b_k\u^{(2,k)}\rp \rp \nonumber  \\
 b_k & \triangleq & b_k(\p,\q)=\sqrt{\p_{k-1}-\p_k} \nonumber \\
c_k & \triangleq & c_k(\p,\q)=\sqrt{\q_{k-1}-\q_k}.
 \end{eqnarray}

\section{Practical implementation of ABP -- sfl RDT connection}
\label{sec:implem}

The above definitions allow to recall on the following  fundamental sfl RDT theorem which is among the key components that will be used to practically establish the above mentioned  ABP -- sfl RDT connection.

\begin{theorem} \cite{Stojnicbinperflrdt23,Stojnicflrdt23}
\label{thm:thmsflrdt1}  Consider large $n$ linear regime with  $\alpha=\lim_{n\rightarrow\infty} \frac{m}{n}$ remaining constant as  $n$ grows. Let $G\in\mR^{m\times n}$ be comprised of  independent standard normals and let $\cX\subseteq \mS^n$ and $\cY\subseteq \mS^m$ be two given sets. Assume the complete sfl RDT frame from \cite{Stojnicsflgscompyx23} and for a given function $f(\y):R^m\rightarrow R$ set
\begin{align}\label{eq:thmsflrdt2eq1a0}
   \psi_{rp} & \triangleq - \max_{\x\in\cX} s \max_{\y\in\cY} \lp \y^TG\x + f(\y) \rp
   \qquad  \mbox{(\bl{\textbf{random primal}})} \nonumber \\
   \psi_{rd}(\p,\q,\c) & \triangleq    \frac{1}{2}    \sum_{k=2}^{r+1}\Bigg(\Bigg.
   \p_{k-1}\q_{k-1}
   -\p_{k}\q_{k}
  \Bigg.\Bigg)
\c_k
  - \psi_{S,\infty}(f(\y),\cX,\cY,\p,\q,\c,s) \quad \mbox{(\bl{\textbf{fl random dual}})}. \nonumber \\
 \end{align}
Let $\hat{\p_0}\rightarrow 1$, $\hat{\q_0}\rightarrow 1$, and $\hat{\c_0}\rightarrow 1$, $\hat{\p}_{r+1}=\hat{\q}_{r+1}=\hat{\c}_{r+1}=0$ and  let the non-fixed parts of $\hat{\p}$, $\hat{\q}$, and  $\hat{\c}$ be the solutions of the following system
\begin{eqnarray}\label{eq:thmsflrdt2eq2a0}
   \frac{d \psi_{rd}(\p,\q,\c)}{d\p} =  0,\quad
   \frac{d \psi_{rd}(\p,\q,\c)}{d\q} =  0,\quad
   \frac{d \psi_{rd}(\p,\q,\c)}{d\c} =  0.
 \end{eqnarray}
 Then,
\begin{eqnarray}\label{eq:thmsflrdt2eq3a0}
    \lim_{n\rightarrow\infty} \frac{\mE_G  \psi_{rp}}{\sqrt{n}}
  & = &
 \lim_{n\rightarrow\infty} \psi_{rd}(\hat{\p},\hat{\q},\hat{\c}) \qquad \mbox{(\bl{\textbf{strong sfl random duality}})},\nonumber \\
 \end{eqnarray}
where $\psi_{S,\infty}(\cdot)$ is as in (\ref{eq:fl4})-(\ref{eq:fl5}).
 \end{theorem}
\begin{proof}
After a cosmetic change $f_S(\cdot)=f(\y)$ the proof follows automatically from Theorem 1 and Corollary 1 in \cite{Stojnicbinperflrdt23}.
 \end{proof}

The conceptual advantage offered by Theorem \ref{thm:thmsflrdt1} is in simplified structure of the so-called random dual. However, one needs to be careful with the interpretation of such an advantage. Namely, no matter how elegant the results of Theorem \ref{thm:thmsflrdt1} may look like, they are ultimately practically relevant only if all the underlying numerical evaluations can be successfully conducted. In general, there are two types of problems one may face. First, sets $\cX$ and $\cY$ in general may not have a component-wise structure which would question  straightforwardness of $\x$ and/or $\y$ decouplings. Second, since \emph{a priori}  $r$  is allowed to be any positive integer, the convergence in $r$ may be sufficiently slow that practical realization of numerical evaluations  (typically dictated by memory requirements)  on higher lifting levels may be infeasible. As we will show below, both of these obstacles do appear. While it eventually turns out that both of them can  be handled in a satisfactory manner, the second one actually poses a substantial challenge.

\subsection{$\cX$, $\cY$, $f(\y)$, and $s$ specializations }
\label{sec:neg}

The results of Theorem \ref{thm:thmsflrdt1} are generic and hold for a wide range of combinations of $\cX$, $\cY$, $f(\y)$, and $s$. To make them applicable in the ABP context, we fix a $\kappa\in\mR$ and note the role of the following specialization:
$\cX=\{-\frac{1}{\sqrt{n}},\frac{1}{\sqrt{n}}\}^n$, $\cY=\mS_+^m$, $f(\y)= \kappa \y^T\1$, and $s=-1$ (with $\1$ denoting the column vector of all ones). After implementation of this specialization the  \emph{random dual} can be rewritten as
\begin{align}\label{eq:prac1}
    \psi_{rd}(\p,\q,\c) & \triangleq    \frac{1}{2}    \sum_{k=2}^{r+1}\Bigg(\Bigg.
   \p_{k-1}\q_{k-1}
   -\p_{k}\q_{k}
  \Bigg.\Bigg)
\c_k
  - \psi_{S,\infty}( \kappa \y^T\1,\cX,\cY,\p,\q,\c,s). \nonumber \\
  & =   \frac{1}{2}    \sum_{k=2}^{r+1}\Bigg(\Bigg.
   \p_{k-1}\q_{k-1}
   -\p_{k}\q_{k}
  \Bigg.\Bigg)
\c_k
  - \frac{1}{n}\varphi(D^{(bin)}(s)) - \frac{1}{n}\varphi(D^{(sph)}(s)), \nonumber \\
  \end{align}
where analogously to (\ref{eq:fl4})-(\ref{eq:fl5})
  \begin{eqnarray}\label{eq:prac2}
\varphi(D,\c) & = &
 \mE_{G,{\mathcal U}_{r+1}} \frac{1}{\c_r} \log
\lp \mE_{{\mathcal U}_{r}} \lp \dots \lp \mE_{{\mathcal U}_3}\lp\lp\mE_{{\mathcal U}_2} \lp
\lp    e^{D}   \rp^{\c_2}\rp\rp^{\frac{\c_3}{\c_2}}\rp\rp^{\frac{\c_4}{\c_3}} \dots \rp^{\frac{\c_{r}}{\c_{r-1}}}\rp, \nonumber \\
  \end{eqnarray}
and
\begin{eqnarray}\label{eq:prac3}
D^{(bin)}(s) & = & \max_{\x\in \{-\frac{1}{\sqrt{n}},\frac{1}{\sqrt{n}}\}^n} \lp   s\sqrt{n}      \lp\sum_{k=2}^{r+1}c_k\h^{(k)}\rp^T\x  \rp \nonumber \\
  D^{(sph)}(s) & \triangleq  &   s \max_{\y\in\mS_+^m}
\lp \sqrt{n} \kappa \y^T\1 + \sqrt{n}  \y^T\lp\sum_{k=2}^{r+1}b_k\u^{(2,k)}\rp \rp.
 \end{eqnarray}
We first have
\begin{eqnarray}\label{eq:prac4}
D^{(bin)}(s)
=  \sum_{i=1}^n D^{(bin)}_i,  \quad \mbox{with}\quad
D^{(bin)}_i(c_k)=\left |\lp\sum_{k=2}^{r+1}c_k\h_i^{(k)}\rp \right |,
\end{eqnarray}
and
  \begin{eqnarray}\label{eq:prac6}
\varphi(D^{(bin)}(s),\c) & = &
n \mE_{G,{\mathcal U}_{r+1}} \frac{1}{\c_r} \log
\lp \mE_{{\mathcal U}_{r}} \lp \dots \lp \mE_{{\mathcal U}_3}\lp\lp\mE_{{\mathcal U}_2} \lp
    e^{\c_2D_1^{(bin)}}  \rp\rp^{\frac{\c_3}{\c_2}}\rp\rp^{\frac{\c_4}{\c_3}} \dots \rp^{\frac{\c_{r}}{\c_{r-1}}}\rp
    = n\varphi(D_1^{(bin)}). \nonumber \\
   \end{eqnarray}
From \cite{Stojnicbinperflrdt23} we also find
\begin{equation} \label{eq:prac9}
   D^{(sph)}(s)
    =    s \min_{\gamma_{sq}} \lp \sum_{i=1}^{m} D_i^{(sph)}(b_k)+\gamma_{sq}n \rp, \quad \mbox{with} \quad
    D_i^{(sph)}(b_k)= \frac{\max \lp \kappa + \sum_{k=2}^{r+1}b_k\u_i^{(2,k)},0  \rp^2}{4\gamma_{sq}}.
 \end{equation}
Specializing further to $s=-1$ allows to establish the following connection between the ground state energy, $f_{sq}(\infty)$, from (\ref{eq:limlogpartfunsqrta0}), and the random primal, $\psi_{rp}(\cdot)$, from Theorem \ref{thm:thmsflrdt1},
 \begin{eqnarray}
-f_{sq}(\infty)
 & = &
- \lim_{n\rightarrow\infty}\frac{\mE_G \max_{\x\in\cX}  -  \max_{\y\in\cY} \lp \y^TG\x + f(\y) \rp  }{\sqrt{n}}
 =
    \lim_{n\rightarrow\infty} \frac{\mE_G  \psi_{rp}}{\sqrt{n}}
   =
 \lim_{n\rightarrow\infty} \psi_{rd}(\hat{\p},\hat{\q},\hat{\c}).
  \label{eq:negprac11}
\end{eqnarray}
 Utilizing (\ref{eq:prac1})-(\ref{eq:negprac11}), one then finds
 \begin{eqnarray}\label{eq:negprac13}
 \lim_{n\rightarrow\infty} \psi_{rd}(\hat{\p},\hat{\q},\hat{\c})
     & = &  \frac{1}{2}    \sum_{k=2}^{r+1}\Bigg(\Bigg.
   \hat{\p}_{k-1}\hat{\q}_{k-1}
   -\hat{\p}_{k}\hat{\q}_{k}
  \Bigg.\Bigg)
\hat{\c}_k
\nonumber \\
& &  - \varphi(D_1^{(bin)}(c_k(\hat{\p},\hat{\q})),\hat{\c}) + \hat{\gamma}_{sq}- \alpha\varphi(-D_1^{(sph)}(b_k(\hat{\p},\hat{\q})),\hat{\c})
\nonumber \\
& \triangleq &   \bar{\psi}_{rd}(\hat{\p},\hat{\q},\hat{\c},\hat{\gamma}_{sq}) .
  \end{eqnarray}
Combining  (\ref{eq:negprac11}) and (\ref{eq:negprac13}) we obtain
 \begin{eqnarray}
-f_{sq}(\infty)
& = &  -\lim_{n\rightarrow\infty}\frac{\mE_G \max_{\x\in\cX}  -  \max_{\y\in\cY} \lp  \y^TG\x + f(\y) \rp  }{\sqrt{n}}
=
 \lim_{n\rightarrow\infty} \psi_{rd}(\hat{\p},\hat{\q},\hat{\c})
 =   \bar{\psi}_{rd}(\hat{\p},\hat{\q},\hat{\c},\hat{\gamma}_{sq}) \nonumber \\
 & = &   \frac{1}{2}    \sum_{k=2}^{r+1}\Bigg(\Bigg.
   \hat{\p}_{k-1}\hat{\q}_{k-1}
   -\hat{\p}_{k}\hat{\q}_{k}
  \Bigg.\Bigg)
\hat{\c}_k
  - \varphi(D_1^{(bin)}(c_k(\hat{\p},\hat{\q})),\hat{\c}) + \hat{\gamma}_{sq} - \alpha\varphi(-D_1^{(sph)}(b_k(\hat{\p},\hat{\q})),\hat{\c}). \nonumber \\
  \label{eq:negprac18}
\end{eqnarray}
 The following theorem summarizes the above considerations.

\begin{theorem}
  \label{thme:negthmprac1}
  Assume the complete sfl RDT setup of \cite{Stojnicsflgscompyx23}. Let $\varphi(\cdot)$ and $\bar{\psi}(\cdot)$ be as in (\ref{eq:prac2}) and (\ref{eq:negprac13}), respectively. Consider large $n$ linear regime with $\alpha=\lim_{n\rightarrow\infty} \frac{m}{n}$. Let the ``fixed'' parts of $\hat{\p}$, $\hat{\q}$, and $\hat{\c}$ satisfy $\hat{\p}_1\rightarrow 1$, $\hat{\q}_1\rightarrow 1$, $\hat{\c}_1\rightarrow 1$, $\hat{\p}_{r+1}=\hat{\q}_{r+1}=\hat{\c}_{r+1}=0$ and let the ``non-fixed'' parts of $\hat{\p}_k$, $\hat{\q}_k$, and $\hat{\c}_k$ ($k\in\{2,3,\dots,r\}$) satisfy
  \begin{eqnarray}\label{eq:negthmprac1eq1}
   \frac{d \bar{\psi}_{rd}(\p,\q,\c,\gamma_{sq})}{d\p} =
   \frac{d \bar{\psi}_{rd}(\p,\q,\c,\gamma_{sq})}{d\q} =
   \frac{d \bar{\psi}_{rd}(\p,\q,\c,\gamma_{sq})}{d\c} =
   \frac{d \bar{\psi}_{rd}(\p,\q,\c,\gamma_{sq})}{d\gamma_{sq}} =  0.
 \end{eqnarray}
For
\begin{eqnarray}\label{eq:prac17}
c_k(\hat{\p},\hat{\q})  & = & \sqrt{\hat{\q}_{k-1}-\hat{\q}_k} \nonumber \\
b_k(\hat{\p},\hat{\q})  & = & \sqrt{\hat{\p}_{k-1}-\hat{\p}_k},
 \end{eqnarray}
one then has
 \begin{eqnarray}
-f_{sq}(\infty)
& = &     \frac{1}{2}    \sum_{k=2}^{r+1}\Bigg(\Bigg.
   \hat{\p}_{k-1}\hat{\q}_{k-1}
   -\hat{\p}_{k}\hat{\q}_{k}
  \Bigg.\Bigg)
\hat{\c}_k
  - \varphi(D_1^{(bin)}(c_k(\hat{\p},\hat{\q})),\hat{\c}) + \hat{\gamma}_{sq} - \alpha\varphi(-D_1^{(sph)}(b_k(\hat{\p},\hat{\q})),\hat{\c}). \nonumber \\
  \label{eq:negthmprac1eq2}
\end{eqnarray}
\end{theorem}
\begin{proof}
Follows automatically from the discussion presented above, Theorem \ref{thm:thmsflrdt1}, and the sfl RDT machinery presented in \cite{Stojnicnflgscompyx23,Stojnicsflgscompyx23,Stojnicflrdt23}.
\end{proof}

\subsection{Numerical evaluations $(r\in\{1,2\})$ -- sfl RDT satisfiability threshold}
\label{sec:nuemrical}

All conceptual ingredients needed to conduct numerical evaluations are present in Theorem \ref{thme:negthmprac1}. To ensure that the progressing of the lifting mechanism is presented in a systematic way, we start the numerical evaluations with $r=1$ and proceed inductively. As we will soon see, we make a strict distinction between $r\in\{1,2\}$ and $r\geq 3$ scenarios. Also, for the purpose of obtaining concrete numerical values, the evaluations are specialized to the most famous, zero-threshold, $\kappa=0$, case.

\subsubsection{$r=1$ -- first level of lifting}
\label{sec:firstlev}

For $r=1$ we first have $\hat{\p}_1\rightarrow 1$ and $\hat{\q}_1\rightarrow 1$. This together with $\hat{\p}_{r+1}=\hat{\p}_{2}=\hat{\q}_{r+1}=\hat{\q}_{2}=0$, and $\hat{\c}_{2}\rightarrow 0$ allows to write
\begin{align}\label{eq:negprac19}
    \bar{\psi}_{rd}(\hat{\p},\hat{\q},\hat{\c},\gamma_{sq})   & =   \frac{1}{2}
\c_2
  - \frac{1}{\c_2}\log\lp \mE_{{\mathcal U}_2} e^{\c_2|\sqrt{1-0}\h_1^{(2)} |}\rp +\gamma_{sq}
- \alpha\frac{1}{\c_2}\log\lp \mE_{{\mathcal U}_2} e^{-\c_2\frac{\max(\kappa+\sqrt{1-0}\u_1^{(2,2)},0)^2}{4\gamma_{sq}}}\rp \nonumber \\
& \rightarrow
  - \frac{1}{\c_2}\log\lp 1+ \mE_{{\mathcal U}_2} \c_2|\sqrt{1-0}\h_1^{(2)} |\rp +\gamma_{sq} \nonumber \\
& \qquad - \alpha\frac{1}{\c_2}\log\lp 1- \mE_{{\mathcal U}_2} \c_2\frac{\max(\kappa+\sqrt{1-0}\u_1^{(2,2)},0)^2}{4\gamma_{sq}}\rp \nonumber \\
& \rightarrow
   - \frac{1}{\c_2}\log\lp 1+ \c_2\sqrt{\frac{2}{\pi}}\rp +\gamma_{sq} \nonumber \\
& \qquad - \alpha\frac{1}{\c_2}\log\lp 1- \frac{\c_2}{4\gamma_{sq}} \mE_{{\mathcal U}_2} \max(\kappa+\sqrt{1-0}\u_1^{(2,2)},0)^2 \rp \nonumber \\
& \rightarrow
  - \sqrt{\frac{2}{\pi}}+\gamma_{sq}
+  \frac{\alpha}{4\gamma_{sq}}\mE_{{\mathcal U}_2} \max(\kappa+\sqrt{1-0}\u_1^{(2,2)},0)^2.
  \end{align}
Optimization over $\gamma_{sq}$ gives $\hat{\gamma}_{sq}=\frac{\sqrt{\alpha}}{2}\sqrt{\mE_{{\mathcal U}_2} \max(\kappa+\sqrt{1-0}\u_1^{(2,2)},0)^2}$ and
\begin{align}\label{eq:negprac20}
 - f_{sq}^{(1)}(\infty)=\bar{\psi}_{rd}(\hat{\p},\hat{\q},\hat{\c},\hat{\gamma}_{sq} )   & =
  - \sqrt{\frac{2}{\pi}}+\sqrt{\alpha}\sqrt{\mE_{{\mathcal U}_2} \max(\kappa+\u_1^{(2,2)},0)^2}.
  \end{align}
Critical capacity estimate on the first level of lifting, $\alpha_c^{(1)}$, is then obtained from the condition $f_{sq}^{(1)}(\infty)=0$. For general $\kappa$ one finds
\begin{equation}\label{eq:negprac20a0}
a_c^{(1)}(\kappa)
=  \frac{2}{\pi\mE_{{\mathcal U}_2} \max(\kappa+\u_1^{(2,2)},0)^2}
=  \frac{2}{\pi\lp \frac{\kappa e^{-\frac{\kappa^2}{2}}}{\sqrt{2\pi}} + \frac{(\kappa^2+1)\erfc\lp -\frac{\kappa}{\sqrt{2}} \rp}{2}  \rp}.
  \end{equation}
Specializing to $\kappa=0$ we then obtain
\begin{equation}\label{eq:negprac21}
\hspace{-1in}(\mbox{\textbf{first level:}}) \qquad \qquad  \alpha_c^{(1)}(0) =
  \frac{2}{\pi\mE_{{\mathcal U}_2} \max(\u_1^{(2,2)},0)^2} =  \frac{2}{\pi\frac{1}{2}} = \frac{4}{\pi}
\approx  \bl{\mathbf{1.2732}}.
  \end{equation}

\subsubsection{$r=2$ -- second level of lifting}
\label{sec:secondlev}

One now has $r=2$, $\hat{\p}_1\rightarrow 1$ and $\hat{\q}_1\rightarrow 1$,  and $\hat{\p}_{r+1}=\hat{\p}_{3}=\hat{\q}_{r+1}=\hat{\q}_{3}=0$. However, in general   $\p_2\neq0$, $\q_2\neq0$, and $\hat{\c}_{2}\neq 0$.  Analogously to (\ref{eq:negprac19}), we write
\begin{eqnarray}\label{eq:negprac24}
    \bar{\psi}_{rd}(\p,\q,\c,\gamma_{sq})   & = &  \frac{1}{2}
(1-\p_2\q_2)\c_2
  - \frac{1}{\c_2}\mE_{{\mathcal U}_3}\log\lp \mE_{{\mathcal U}_2} e^{\c_2|\sqrt{1-\q_2}\h_1^{(2)} +\sqrt{\q_2}\h_1^{(3)} |}\rp \nonumber \\
& &   + \gamma_{sq}
 -\alpha\frac{1}{\c_2}\mE_{{\mathcal U}_3} \log\lp \mE_{{\mathcal U}_2} e^{-\c_2\frac{\max(\sqrt{1-\p_2}\u_1^{(2,2)}+\sqrt{\p_2}\u_1^{(2,3)} + \kappa ,0)^2}{4\gamma_{sq}}}\rp.
    \end{eqnarray}
After computing the inner integral we find
\begin{eqnarray}\label{eq:negprac24a0}
f_{(z)}^{(2)} & = & \mE_{{\mathcal U}_2} e^{\c_2|\sqrt{1-\q_2}\h_1^{(2)} +\sqrt{\q_2}\h_1^{(3)} |}
  \nonumber \\
 & = &  \frac{1}{2}
 e^{\frac{(1-\q_2)\c_2^2}{2}}
 \Bigg(\Bigg.
 e^{-\c_2\sqrt{\q_2}\h_1^{(3)}}
 \erfc\lp - \lp\c_2\sqrt{1-\q_2}-\frac{\sqrt{\q_2}\h_1^{(3)}}{\sqrt{1-\q_2}}\rp\frac{1}{\sqrt{2}}\rp \nonumber \\
& &  + e^{\c_2\sqrt{\q_2}\h_1^{(3)}}
   \erfc\lp - \lp\c_2\sqrt{1-\q_2}+\frac{\sqrt{\q_2}\h_1^{(3)}}{\sqrt{1-\q_2}}\rp\frac{1}{\sqrt{2}}\rp
   \Bigg.\Bigg),
     \end{eqnarray}
and
\begin{eqnarray}\label{eq:negprac24a1}
  \mE_{{\mathcal U}_3}\log\lp \mE_{{\mathcal U}_2} e^{\c_2|\sqrt{1-\q_2}\h_1^{(2)} +\sqrt{\q_2}\h_1^{(3)} |}\rp
=  \mE_{{\mathcal U}_3}\log\lp f_{(z)}^{(2)}\rp.
    \end{eqnarray}
In a similar fashion, we also obtain
\begin{eqnarray}\label{eq:negprac24a2}
\bar{h} & = &  -\frac{\sqrt{\p_2}\u_1^{(2,3)}+\kappa}{\sqrt{1-\p_2}}    \nonumber \\
\bar{B} & = & \frac{\c_2}{4\gamma_{sq}} 
\nonumber \\
\bar{C} & = & \sqrt{\p_2}\u_1^{(2,3)}+\kappa \nonumber \\
f_{(zd)}^{(2,f)}& = & \frac{e^{-\frac{\bar{B}\bar{C}^2}{2(1-\p_2)\bar{B} + 1}}}{2\sqrt{2(1-\p_2)\bar{B} + 1}}
\erfc\lp\frac{\bar{h}}{\sqrt{4(1-\p_2)\bar{B} + 2}}\rp
\nonumber \\
f_{(zu)}^{(2,f)}& = & \frac{1}{2}\erfc\lp-\frac{\bar{h}}{\sqrt{2}}\rp,  
   \end{eqnarray}
and
\begin{eqnarray}\label{eq:negprac24a3}
   \mE_{{\mathcal U}_3} \log\lp \mE_{{\mathcal U}_2} e^{-\c_2\frac{\max(\sqrt{1-\p_2}\u_1^{(2,2)}+\sqrt{\p_2}\u_1^{(2,3)} +\kappa  ,0)^2}{4\gamma_{sq}}}\rp
=   \mE_{{\mathcal U}_3} \log\lp f_{(zd)}^{(2,f)}+f_{(zu)}^{(2,f)}\rp.
    \end{eqnarray}
Differentiating  (optimizing) with respect to $\gamma_{sq}$, $\p_2$, $\q_2$, and $\c_2$, one further finds $\c_2\rightarrow \infty$, $\gamma_{sq}\rightarrow 0$,
$\q_2\c_2^2\rightarrow \q_2^{(s)}$, and
\begin{eqnarray}\label{eq:negprac24a4}
f_{(z)}^{(2)}
 & \rightarrow &
 e^{\frac{\c_2^2-\q_2^{(s)}}{2}}
 \Bigg(\Bigg.
 e^{-\sqrt{\q_2^{(s)}}\h_1^{(3)}}
    + e^{\sqrt{\q_2^{(s)}}\h_1^{(3)}}
    \Bigg.\Bigg).
     \end{eqnarray}
The above allows to transform (\ref{eq:negprac24a2})  into
\begin{eqnarray}\label{eq:negprac24a5}
\bar{h} & = &  -\frac{\sqrt{\p_2}\u_i^{(2,3)}+\kappa}{\sqrt{1-\p_2}}    \nonumber \\
\bar{B} & = & \frac{\c_2}{4\gamma_{sq}} \rightarrow \infty
\nonumber \\
\bar{C} & = & \sqrt{\p_2}\u_i^{(2,3)}+\kappa \nonumber \\
f_{(zd)}^{(2,f)}& = & \frac{e^{-\frac{\bar{B}\bar{C}^2}{2(1-\p_2)\bar{B} + 1}}}{2\sqrt{2(1-\p_2)\bar{B} + 1}}
\erfc\lp\frac{\bar{h}}{\sqrt{4(1-\p_2)\bar{B} + 2}}\rp
\rightarrow 0
\nonumber \\
f_{(zu)}^{(2,f)}& = & \frac{1}{2}\erfc\lp-\frac{\bar{h}}{\sqrt{2}}\rp
\rightarrow \frac{1}{2}\erfc\lp \frac{\sqrt{\p_2}\u_i^{(2,3)}+\kappa}{\sqrt{2}\sqrt{1-\p_2}} \rp.
   \end{eqnarray}
Finally, a combination of (\ref{eq:negprac24}), (\ref{eq:negprac24a1}), (\ref{eq:negprac24a3}), (\ref{eq:negprac24a4}), and (\ref{eq:negprac24a5}) gives
\begin{eqnarray}\label{eq:negprac24a6}
-f_{sq}^{(2,f)}(\infty) & = &     \bar{\psi}_{rd}(\p,\q,\c,\gamma_{sq}) \nonumber \\
  & = &  \frac{1}{2}
(1-\p_2\q_2)\c_2
  - \frac{1}{\c_2}\mE_{{\mathcal U}_3}\log\lp \mE_{{\mathcal U}_2} e^{\c_2|\sqrt{1-\q_2}\h_1^{(2)} +\sqrt{\q_2}\h_1^{(3)} |}\rp \nonumber \\
& &   + \gamma_{sq}
 -\alpha\frac{1}{\c_2}\mE_{{\mathcal U}_3} \log\lp \mE_{{\mathcal U}_2} e^{-\c_2\frac{\max(\sqrt{1-\p_2}\u_1^{(2,2)}+\sqrt{\p_2}\u_1^{(2,3)} +\kappa ,0)^2}{4\gamma_{sq}}}\rp \nonumber \\
 & = &  \frac{1}{2}
(1-\p_2\q_2)\c_2
  - \frac{1}{\c_2}\mE_{{\mathcal U}_3}\log\lp f_{(z)}^{(2)} \rp   + \gamma_{sq}
 -\alpha\frac{1}{\c_2}\mE_{{\mathcal U}_3} \log\lp f_{(zd)}^{(2,f)}+f_{(zu)}^{(2,f)}\rp \nonumber\\
 & \rightarrow & \frac{1}{2}
(1-\p_2\q_2)\c_2
  - \frac{1}{\c_2}\mE_{{\mathcal U}_3}\log\lp
 e^{\frac{\c_2^2-\q_2^{(s)}}{2}}
 \Bigg(\Bigg.
 e^{-\sqrt{\q_2^{(s)}}\h_1^{(3)}}
    + e^{\sqrt{\q_2^{(s)}}\h_1^{(3)}}
    \Bigg.\Bigg) \rp  \nonumber \\
& &
 -\alpha\frac{1}{\c_2}\mE_{{\mathcal U}_3} \log\lp   \frac{1}{2}\erfc\lp \frac{\sqrt{\p_2}\u_1^{(2,3)}+\kappa}{\sqrt{2}\sqrt{1-\p_2}} \rp  \rp \nonumber\\
 & \rightarrow & \frac{1}{2}
\frac{(1-\p_2)\q_2^{(s)}}{\c_2}
  - \frac{1}{\c_2}\mE_{{\mathcal U}_3}\log\lp 2 \cosh \lp \sqrt{\q_2^{(s)}}\h_1^{(3)}\rp \rp
  \nonumber \\
& &
 -\alpha\frac{1}{\c_2}\mE_{{\mathcal U}_3} \log\lp   \frac{1}{2}\erfc\lp \frac{\sqrt{\p_2}\u_1^{(2,3)}+\kappa}{\sqrt{2}\sqrt{1-\p_2}} \rp  \rp. \nonumber\\
    \end{eqnarray}
After computing all the derivatives and specializing to $\kappa=0$, we obtain for the second lifting level  critical capacity estimate
\begin{equation}\label{eq:negprac25}
\hspace{-2in}(\mbox{\textbf{second level:}}) \qquad \qquad  \alpha_c^{(2)}(0) \approx
 \bl{\mathbf{0.8331}}.
  \end{equation}
All relevant parametric values for the first (1-sfl RDT) and second  (2-sfl RDT) level of lifting are systematically shown in Table \ref{tab:tab1}. Under the assumption that sequence $\c$ is decreasing, i.e., under the assumption that
\begin{eqnarray}
\label{eq:addalg1}
1= \c_1 \geq \c_2 \geq \c_3\geq \dots\geq \c_{r+1}=0,
\end{eqnarray}
we find no further changes for $r\geq 3$. This effectively reconfirms that $\alpha_c^{(2)}(0) \approx
  0.8331$ is indeed the ABP's satisfiability threshold (precisely as obtained in \cite{KraMez89,DingSun19,Huang24,Stojnicbinperflrdt23}).

\begin{table}[h]
\caption{$r$-sfl RDT parameters ($r\leq 2$);
 $\kappa=0$; $n,\beta\rightarrow\infty$}\vspace{.1in}
\centering
\def\arraystretch{1.2}
\begin{tabular}{||l||c||c|c||c|c||c||c||}\hline\hline
 \hspace{-0in}$r$-sfl RDT                                             & $\hat{\gamma}_{sq}$    &  $\hat{\p}_2$ & $\hat{\p}_1`$     & $\hat{\q}_2^{(s)}\rightarrow \hat{\q}_2\hat{\c}_2^2$  & $\hat{\q}_1$ &  $\hat{\c}_2$    & $\alpha_c^{(r)}(0)$  \\ \hline\hline
$1$-sfl RDT                                      & $0.3989$ &  $0$  & $\rightarrow 1$   & $0$ & $\rightarrow 1$
 &  $\rightarrow 0$  & \bl{$\mathbf{1.2732}$} \\ \hline
   $2$-sfl RDT                                      & $0$  & $0.5639$ & $\rightarrow 1$ &  $2.5764$ & $\rightarrow 1$
 &  $\rightarrow \infty$   & \bl{$\mathbf{0.8331}$}  \\ \hline\hline
  \end{tabular}
\label{tab:tab1}
\end{table}

\subsection{Numerical evaluations $(r\geq 3)$ -- parametric sfl RDT  algorithmic threshold implications}
\label{sec:numericalr3}

In this section we propose a seemingly unconventional approach. Namely, instead of imposing natural decreasing order for sequence $\c$ as in (\ref{eq:addalg1}), we now remove such a restriction and for $r\geq 3$ look at any nonnegative real sequence $\c$.

\subsubsection{$r=3$ -- third level of lifting}
\label{sec:thirdlev}

 For $r=3$, $\hat{\p}_1\rightarrow 1$ and $\hat{\q}_1\rightarrow 1$,  and $\hat{\p}_{r+1}=\hat{\p}_{4}=\hat{\q}_{r+1}=\hat{\q}_{4}=0$, but in general   $\p_2\neq0$,  $\p_3\neq0$, $\q_2\neq0$, $\q_3\neq0$,  $\hat{\c}_{2}\neq 0$, and $\hat{\c}_{3}\neq 0$.  Analogously to (\ref{eq:negprac19}) and (\ref{eq:negprac24}), one writes
\begin{eqnarray}\label{eq:algnegprac24}
    \bar{\psi}_{rd}(\p,\q,\c,\gamma_{sq})   & = &
     \frac{1}{2}
(1-\p_2\q_2)\c_2
 +  \frac{1}{2}
(\p_2\q_2-\p_3\q_3)\c_3
\nonumber \\
& &
  - \frac{1}{\c_3}\mE_{{\mathcal U}_4}\log \lp  \mE_{{\mathcal U}_3}  \lp \mE_{{\mathcal U}_2} e^{\c_2|\sqrt{1-\q_2}\h_1^{(2)} +\sqrt{\q_2-\q_3}\h_1^{(3)}+\sqrt{\q_3}\h_1^{(4)}   |}\rp^{\frac{\c_3}{\c_2}}  \rp \nonumber \\
& &   + \gamma_{sq}
 -\alpha\frac{1}{\c_3}\mE_{{\mathcal U}_4} \log \lp \mE_{{\mathcal U}_3}   \lp \mE_{{\mathcal U}_2} e^{-\c_2\frac{\max(\sqrt{1-\p_2}\u_1^{(2,2)}+\sqrt{\p_2-\p_3}\u_1^{(2,3)}+\sqrt{\p_3}\u_1^{(2,4)} +\kappa ,0)^2}{4\gamma_{sq}}}\rp^{\frac{\c_3}{\c_2}} \rp.\nonumber \\
    \end{eqnarray}
We first  compute the inner integral and find
\begin{eqnarray}\label{eq:algnegprac24a0}
f_{(z)}^{(3)} & = & \mE_{{\mathcal U}_2} e^{\c_2|\sqrt{1-\q_2}\h_1^{(2)} +\sqrt{\q_2-\q_3}\h_1^{(3)} +\sqrt{\q_3}\h_1^{(4)}  |}
  \nonumber \\
 & = &  \frac{1}{2}
 e^{\frac{(1-\q_2)\c_2^2}{2}}
 \Bigg(\Bigg.
 e^{-\c_2  ( \sqrt{\q_2-\q_3}\h_1^{(3)}  + \sqrt{\q_3}\h_1^{(4)}     ) }
 \erfc\lp - \lp\c_2\sqrt{1-\q_2}-\frac{\sqrt{\q_2-\q_3}\h_1^{(3)} + \sqrt{\q_3}\h_1^{(4)}  } {\sqrt{1-\q_2}}\rp\frac{1}{\sqrt{2}}\rp \nonumber \\
& &  + e^{\c_2    ( \sqrt{\q_2-\q_3}\h_1^{(3)}  + \sqrt{\q_3}\h_1^{(4)}     )  }
   \erfc\lp - \lp\c_2\sqrt{1-\q_2}+\frac{\sqrt{\q_2-\q_3}\h_1^{(3)} + \sqrt{\q_3}\h_1^{(4)} } {\sqrt{1-\q_2}}\rp\frac{1}{\sqrt{2}}\rp
   \Bigg.\Bigg)
     \nonumber \\
 & = &  \frac{1}{2}
 e^{\frac{(1-\q_2)\c_2^2}{2}}
 \Bigg(\Bigg.
 e^{-\c_2\zeta_3  }
 \erfc\lp - \lp\c_2\sqrt{1-\q_2}-\frac{\zeta_3 } {\sqrt{1-\q_2}}\rp\frac{1}{\sqrt{2}}\rp \nonumber \\
& &  + e^{\c_2\zeta_3   }
   \erfc\lp - \lp\c_2\sqrt{1-\q_2}+\frac{\zeta_3 } {\sqrt{1-\q_2}}\rp\frac{1}{\sqrt{2}}\rp
   \Bigg.\Bigg)   ,
\end{eqnarray}
where
\begin{eqnarray} \label{eq:algnegprac24a0a0}
\zeta_3 = \sqrt{\q_2-\q_3}\h_1^{(3)} + \sqrt{\q_3}\h_1^{(4)}.
\end{eqnarray}
Consequently
\begin{eqnarray}\label{eq:algnegprac24a1}
  \mE_{{\mathcal U}_4}\log   \lp \mE_{{\mathcal U}_3}  \lp \mE_{{\mathcal U}_2} e^{\c_2|\sqrt{1-\q_2}\h_1^{(2)} +\sqrt{\q_2-\q_3}\h_1^{(3)}  +\sqrt{\q_3}\h_1^{(4)}  |}\rp^{\frac{\c_3}{\c_2}} \rp
=  \mE_{{\mathcal U}_4}\log  \lp  \mE_{{\mathcal U}_3}   \lp f_{(z)}^{(3)}\rp^{\frac{\c_3}{\c_2}} \rp.
    \end{eqnarray}
Following a similar path we also find
\begin{eqnarray}\label{eq:algnegprac24a2}
\eta_3 & = & \sqrt{\p_2-\p_3}\u_1^{(2,3)} + \sqrt{\p_3}\u_1^{(2,4)}    \nonumber \\
\bar{h}_3 & = &  -\frac{\sqrt{\p_2-\p_3}\u_1^{(2,3)}+\sqrt{\p_3}\u_1^{(2,4)}+\kappa}{\sqrt{1-\p_2}} = -\frac{\eta_3+\kappa}{\sqrt{1-\p_2}}    \nonumber \\
\bar{B} & = & \frac{\c_2}{4\gamma_{sq}} 
\nonumber \\
\bar{C}_3 & = & \sqrt{\p_2-\p_3}\u_1^{(2,3)}+\sqrt{\p_3}\u_1^{(2,4)}+\kappa   = \eta_3+\kappa \nonumber \\
f_{(zd)}^{(3,f)}& = & \frac{e^{-\frac{\bar{B}\bar{C}_3^2}{2(1-\p_2)\bar{B} + 1}}}{2\sqrt{2(1-\p_2)\bar{B} + 1}}
\erfc\lp\frac{\bar{h}_3}{\sqrt{4(1-\p_2)\bar{B} + 2}}\rp
\nonumber \\
f_{(zu)}^{(3,f)}& = & \frac{1}{2}\erfc\lp-\frac{\bar{h}_3}{\sqrt{2}}\rp,  
   \end{eqnarray}
and
\begin{equation} \label{eq:algnegprac24a3}
   \mE_{{\mathcal U}_4} \log  \lp \mE_{{\mathcal U}_3} \lp \mE_{{\mathcal U}_2} e^{-\c_2\frac{\max(\sqrt{1-\p_2}\u_1^{(2,2)}+\sqrt{\p_2-\p_3}\u_1^{(2,3)} +\sqrt{\p_3}\u_1^{(2,4)} +  \kappa  ,0)^2}{4\gamma_{sq}}}\rp^{\frac{\c_3}{\c_2}} \rp
=   \mE_{{\mathcal U}_4} \log   \lp \mE_{{\mathcal U}_3}  \lp f_{(zd)}^{(3,f)}+f_{(zu)}^{(3,f)}\rp^{\frac{\c_3}{\c_2}} \rp .
    \end{equation}
Analogously to second lifting level, we now have $\c_2\rightarrow \infty$, $\gamma_{sq}\rightarrow 0$,
$\q_2\c_2^2\rightarrow \q_2^{(s)}$, $\q_3\c_2^2\rightarrow \q_3^{(s)}$, $\frac{\c_3}{\c_2}\rightarrow \c_3^{(s)} $ and
\begin{eqnarray}\label{eq:algnegprac24a4}
f_{(z)}^{(3)}
&  \rightarrow &
 e^{\frac{\c_2^2-\q_2^{(s)}}{2}}
 \Bigg(\Bigg.
 e^{- \lp  \sqrt{\q_2^{(s)} - \q_3^{(s)} }\h_1^{(3)}   +  \sqrt{\q_3^{(s)}}\h_1^{(4)}      \rp }
    + e^{\sqrt{\q_2^{(s)} - \q_3^{(s)} }\h_1^{(3)}   +  \sqrt{\q_3^{(s)}}\h_1^{(4)}   }
    \Bigg.\Bigg)
\nonumber \\
&  =  &
   e^{\frac{\c_2^2-\q_2^{(s)}}{2}}
 \Bigg(\Bigg.
 e^{- \zeta_3^{(s)} }
    + e^{\zeta_3^{(s)}   }
    \Bigg.\Bigg)
      =
  e^{\frac{\c_2^2-\q_2^{(s)}}{2}}
 \Bigg(\Bigg.
  2\cosh \lp \zeta_3^{(s)} \rp
    \Bigg.\Bigg)  ,
\end{eqnarray}
where
\begin{eqnarray}\label{eq:algnegprac24a4a0}
\zeta_3^{(s)}  = \sqrt{\q_2^{(s)} - \q_3^{(s)} }\h_1^{(3)}   +  \sqrt{\q_3^{(s)}}\h_1^{(4)} .
\end{eqnarray}
The above also allows to transform (\ref{eq:algnegprac24a2})  into
\begin{eqnarray}\label{eq:algnegprac24a5}
\bar{h}_3 & = &  -\frac{\eta_3+\kappa}{\sqrt{1-\p_2}}    \nonumber \\
\bar{B} & = & \frac{\c_2}{4\gamma_{sq}} \rightarrow \infty
\nonumber \\
\bar{C}_3 & = & \eta_3+\kappa \nonumber \\
f_{(zd)}^{(3,f)}& = & \frac{e^{-\frac{\bar{B}\bar{C}_3^2}{2(1-\p_2)\bar{B} + 1}}}{2\sqrt{2(1-\p_2)\bar{B} + 1}}
\erfc\lp\frac{\bar{h}_3}{\sqrt{4(1-\p_2)\bar{B} + 2}}\rp
\rightarrow 0
\nonumber \\
f_{(zu)}^{(3,f)}& = & \frac{1}{2}\erfc\lp-\frac{\bar{h}_3}{\sqrt{2}}\rp
\rightarrow \frac{1}{2}\erfc\lp \frac{\eta_3+ \kappa}{\sqrt{2}\sqrt{1-\p_2}} \rp.
   \end{eqnarray}
Combining (\ref{eq:algnegprac24}), (\ref{eq:algnegprac24a1}), (\ref{eq:algnegprac24a3}), (\ref{eq:algnegprac24a4}), and (\ref{eq:algnegprac24a5}) we obtain
\begin{eqnarray}\label{eq:algnegprac24a6}
-f_{sq}^{(3,f)}(\infty) & = &     \bar{\psi}_{rd}(\p,\q,\c,\gamma_{sq})
\nonumber \\
 & = &
     \frac{1}{2}
(1-\p_2\q_2)\c_2
 +  \frac{1}{2}
(\p_2\q_2-\p_3\q_3)\c_3
\nonumber \\
& &
  - \frac{1}{\c_3}\mE_{{\mathcal U}_4}\log \lp  \mE_{{\mathcal U}_3}  \lp \mE_{{\mathcal U}_2} e^{\c_2|\sqrt{1-\q_2}\h_1^{(2)} +\sqrt{\q_2-\q_3}\h_1^{(3)}+\sqrt{\q_3}\h_1^{(4)}   |}\rp^{\frac{\c_3}{\c_2}}  \rp \nonumber \\
& &   + \gamma_{sq}
 -\alpha\frac{1}{\c_3}\mE_{{\mathcal U}_4} \log \lp \mE_{{\mathcal U}_3}   \lp \mE_{{\mathcal U}_2} e^{-\c_2\frac{\max(\sqrt{1-\p_2}\u_1^{(2,2)}+\sqrt{\p_2-\p_3}\u_1^{(2,3)}+\sqrt{\p_3}\u_1^{(2,4)} +\kappa ,0)^2}{4\gamma_{sq}}}\rp^{\frac{\c_3}{\c_2}} \rp \nonumber \\
 & = &      \frac{1}{2}
(1-\p_2\q_2)\c_2
 +  \frac{1}{2}
(\p_2\q_2-\p_3\q_3)\c_3
  - \frac{1}{\c_3}\mE_{{\mathcal U}_4}\log \lp \mE_{{\mathcal U}_3}   \lp f_{(z)}^{(3)} \rp^{\frac{\c_3}{\c_2}} \rp
  \nonumber \\
& &   + \gamma_{sq}
 -\alpha\frac{1}{\c_3}\mE_{{\mathcal U}_4} \log \lp \mE_{{\mathcal U}_3}     \lp f_{(zd)}^{(3,f)}+f_{(zu)}^{(3,f)}\rp^{\frac{\c_3}{\c_2}} \rp   \nonumber\\
 & \rightarrow & \frac{1}{2}
(1-\p_2\q_2)\c_2
 +  \frac{1}{2}
(\p_2\q_2-\p_3\q_3)\c_3
  - \frac{1}{\c_3}\mE_{{\mathcal U}_4}\log  \lp \mE_{{\mathcal U}_3}   \lp
 e^{\frac{\c_2^2-\q_2^{(s)}}{2}}
 \Bigg(\Bigg.
  2\cosh \lp \zeta_3^{(s)} \rp
    \Bigg.\Bigg) \rp^{\frac{\c_3}{\c_2}}  \rp
     \nonumber \\
& &
 -\alpha\frac{1}{\c_3}  \mE_{{\mathcal U}_4} \log  \lp \mE_{{\mathcal U}_3}  \lp   \frac{1}{2}\erfc\lp \frac{\eta_3+\kappa}{\sqrt{2}\sqrt{1-\p_2}} \rp  \rp^{\frac{\c_3}{\c_2}} \rp
 \nonumber\\
 & \rightarrow & \frac{1}{2}
\frac{(1-\p_2)\q_2^{(s)}}{\c_2}
 +  \frac{1}{2}
(\p_2\q_2-\p_3\q_3)\c_3
  - \frac{1}{\c_3}\mE_{{\mathcal U}_4}\log  \lp \mE_{{\mathcal U}_3}   \lp
   2\cosh \lp \zeta_3^{(s)} \rp
  \rp^{\frac{\c_3}{\c_2}}  \rp
  \nonumber \\
& &
 -\alpha\frac{1}{\c_3}  \mE_{{\mathcal U}_4} \log  \lp \mE_{{\mathcal U}_3}  \lp   \frac{1}{2}\erfc\lp \frac{\eta_3+\kappa}{\sqrt{2}\sqrt{1-\p_2}} \rp  \rp^{\frac{\c_3}{\c_2}} \rp
  \nonumber
  \\
& \rightarrow & \frac{1}{2}
\frac{(1-\p_2)\q_2^{(s)}}{\c_2}
 +  \frac{1}{2}
\frac{(\p_2\q_2^{(s)}-\p_3\q_3^{(s)})\c_3^{(s)}}{\c_2}
  - \frac{1}{\c_2\c_3^{(s)}}\mE_{{\mathcal U}_4}\log  \lp \mE_{{\mathcal U}_3}   \lp
  2\cosh \lp \zeta_3^{(s)} \rp
     \rp^{\c_3^{(s)}}  \rp
  \nonumber \\
& &
 -\alpha\frac{1}{\c_2\c_3^{(s)}}  \mE_{{\mathcal U}_4} \log  \lp \mE_{{\mathcal U}_3}  \lp   \frac{1}{2}\erfc\lp \frac{\eta_3+\kappa}{\sqrt{2}\sqrt{1-\p_2}} \rp  \rp^{\c_3^{(s)}} \rp .
    \end{eqnarray}
After setting
\begin{eqnarray}\label{eq:algnegprac24a6a0}
   \bar{\psi}_{rd}^{(3,s)}(\p,\q,\c,\gamma_{sq})
\hspace{-.0in}  & \triangleq & \hspace{-.0in}
  \frac{1}{2}
(1-\p_2)\q_2^{(s)}
 +  \frac{1}{2}
\lp \p_2\q_2^{(s)}-\p_3\q_3^{(s)} \rp\c_3^{(s)}
  - \frac{1}{\c_3^{(s)}}\mE_{{\mathcal U}_4}\log  \lp \mE_{{\mathcal U}_3}   \lp
  2\cosh \lp \zeta_3^{(s)} \rp
     \rp^{\c_3^{(s)}}  \rp
  \nonumber \\
\hspace{-.0in} & & \hspace{-.0in}
 -\alpha\frac{1}{\c_3^{(s)}}  \mE_{{\mathcal U}_4} \log  \lp \mE_{{\mathcal U}_3}  \lp   \frac{1}{2}\erfc\lp \frac{\eta_3+\kappa}{\sqrt{2}\sqrt{1-\p_2}} \rp  \rp^{\c_3^{(s)}} \rp ,
    \end{eqnarray}
one recognizes that condition $f_{sq}(\infty)=0$ used to determine the critical capacity now transforms into $  \bar{\psi}_{rd}^{(3,s)}(\p,\q,\c,\gamma_{sq}) = 0$. Computation of all the derivatives and specialization to $\kappa=0$ give so to say \emph{virtual}
third lifting level capacity estimate
\begin{equation}\label{eq:algnegprac25}
\hspace{-2in}(\mbox{\textbf{third level:}}) \qquad \qquad  \alpha_c^{(3)}(0) \approx
 \red{\mathbf{0.7843}}.
  \end{equation}
All relevant parametric values for the first, second, and third level of lifting (1,2,3-sfl RDT) are systematically shown in Table \ref{tab:tab2}.
\begin{table}[h]
\caption{$r$-sfl RDT parameters  ($r\leq 3$);   $\hat{\c}_2\rightarrow \infty$;   $\hat{\c}_3^{(s)} = \lim_{\hat{\c}_2\rightarrow\infty} \frac{\hat{\c}_3}{\hat{\c}_2}$; $\kappa=0$; $n,\beta\rightarrow\infty$}\vspace{.1in}
\centering
\def\arraystretch{1.2}
{\small \begin{tabular}{||l||c||c|c|c||c|c|c||c|c||c||}\hline\hline
 \hspace{-0in}$r$-sfl RDT                                             & $\hat{\gamma}_{sq}$    &  $\hat{\p}_3$  &  $\hat{\p}_2$ & $\hat{\p}_1$     & $\hat{\q}_3^{(s)}\rightarrow \hat{\q}_3\hat{\c}_2^2$  & $\hat{\q}_2^{(s)}\rightarrow \hat{\q}_2\hat{\c}_2^2$  & $\hat{\q}_1$ &  $\hat{\c}_3^{(s)}$  &  $\hat{\c}_2$    & $\alpha_c^{(r)}(0)$  \\ \hline\hline
$1$-sfl RDT                                      & $0.3989$ &  $0$  &  $0$   & $\rightarrow 1$  &  $0$    & $0$ & $\rightarrow 1$
& $ \rightarrow 0 $ &  $\rightarrow 0$  & \bl{$\mathbf{1.2732}$} \\ \hline
   $2$-sfl RDT                                      & $0$ &  $0$   & $0.5639$ & $\rightarrow 1$ &   $0$  & $2.5764$ & $\rightarrow 1$
& $ \rightarrow 0 $ &  $\rightarrow \infty$   & \bl{$\mathbf{0.8331}$}  \\ \hline\hline
   $3$-sfl RDT                                      & $0$ &  $0.6478$    & $0.9844$ & $\rightarrow 1$ &  $0.2479$   &  $1.0212$ & $\rightarrow 1$
& $4.33$  &  $\rightarrow \infty$   & \red{$\mathbf{0.7843}$}  \\ \hline\hline
  \end{tabular}}
\label{tab:tab2}
\end{table}

As Table \ref{tab:tab2} shows, $\frac{\c_3}{\c_2}=\c_3^{(s)}>1$ and under the natural assumption that sequence $\c$ is decreasing, the obtained capacity estimate would not be considered physical. Even though it is not related to the ABP capacity or satisfiability threshold per se, it does have two interesting properties: (i)  $\alpha_c^{(3)}(0) < \alpha_c^{(2)}(0) $; and (ii) $\alpha_c^{(3)}(0)$ might be close to a range around $\sim 0.78$. Both of these features seem to be tightly connected to ABP's algorithmic properties. Namely, currently the best available polynomial (or in general fast) algorithms can solve ABP instances for constraint densities up to $\sim 0.77$ which suggests that ABP might exhibit a statistical-computational gap. Available theoretical analyses seem to point in this direction as well. In particular, local entropy studies related to the clustering structure of atypical ABP solutions \cite{Bald15,Bald16,Bald21,Stojnicabple25}   point towards $\alpha$ interval $(0.77,0.78)$ as the range where clustering defragmentation (postulated as a likely cause for failure of locally improving algorithms) happens.

Another striking parallel comes from recent progress in algorithmic studying of SK models. In a remarkable breakthrough \cite{Montanari19}, Montanari showed that the ground state energy (GSE) of the pure $2$-spin SK Ising model can be computed  in polynomial time  via IAMP (an incremental modification of AMP) provided that the corresponding Parisi functional is continuously increasing (several other fast algorithms appeared as well \cite{Dandietal25,Erba24,Boet05,Das25,Stojnicclupsk25} achieving similar performance and effectively reaffirming that determining classical SK's GSE is computationally doable in polynomial time). Differently from the classical ($2$-spin) Ising SK model which is expected to be solvable in polynomial time, higher $p$-spin variants might not be. Employing the same IAMP  \cite{ElAlMont20} demonstrated that a statistical-computational gap is indeed likely to happen already for $p=3$ (for all practical purposes it is rather small but it brings a theoretical value as it may reflect itself in a more significant way in pother problems). Moreover, the performance of the introduced IAMP seems to be in an excellent agreement with the virtual GSE estimate obtained after removal of the restrictive  nondecreasing (physical) nature of Parisi functional.

While we operate here in a completely different realm (instead of GSE we consider constraint density of a satisfiability problem and instead of PDE and associated functionals we consider fl RDT and sequences of parameters), we believe that the concepts that we propose extend far beyond ABP and might likely be a consequence of more universal principles. Such principles might then be common for different types of problems as well. A first step to further support and  strengthen our belief is to check whether $\alpha_c^{(r)}(0)$ indeed remains close to $(0.77-0.78)$ range even for higher lifting levels $r$. In general, this is not necessarily an easy numerical task. However, the above results from the third level allow for efficient $r$-level generalization that in return can help to a degree with the residual numerical work.

\subsubsection{General $r$--th level of lifting}
\label{sec:rthlev}

We start by observing that for general $r$ one has  $\hat{\p}_1\rightarrow 1$ and $\hat{\q}_1\rightarrow 1$,  and $\hat{\p}_{r+1}=\hat{\q}_{r+1}=0$, but     $\hat{\p}_k\neq0$, $\hat{\q}_k\neq0$, and $\hat{\c}_{k}\neq 0$ for $2\leq k\leq r$.  Analogously to (\ref{eq:algnegprac24}) (and earlier (\ref{eq:negprac19}) and (\ref{eq:negprac24})), we first write
\begin{eqnarray}\label{eq:algalgnegprac24}
    \bar{\psi}_{rd}(\p,\q,\c,\gamma_{sq})   & = &
     \frac{1}{2}
\sum_{k=2}^{r+1}(\p_{k-1}\q_{k-1}-\p_k\q_k)\c_k
 \nonumber \\
& &
  - \frac{1}{\c_r}\mE_{{\mathcal U}_{r+1}}\log \lp \dots \mE_{{\mathcal U}_4}  \lp  \mE_{{\mathcal U}_3}  \lp \mE_{{\mathcal U}_2} e^{\c_2|  \sum_{k=2}^{r+1} c_k\h_1^{(k)}    |}\rp^{\frac{\c_3}{\c_2}}  \rp^{\frac{\c_4}{\c_3}} \dots \rp    \nonumber \\
& &   + \gamma_{sq}
 -\alpha\frac{1}{\c_r}\mE_{{\mathcal U}_{r+1}} \log \lp \dots \mE_{{\mathcal U}_4}  \lp \mE_{{\mathcal U}_3}   \lp \mE_{{\mathcal U}_2} e^{-\c_2\frac{\max(\sum_{k=2}^{r+1} b_k \u_1^{(2,k)} +\kappa ,0)^2}{4\gamma_{sq}}}\rp^{\frac{\c_3}{\c_2}} \rp^{\frac{\c_4}{\c_3}} \dots \rp ,\nonumber \\
    \end{eqnarray}
where, as in (\ref{eq:fl5}),
\begin{eqnarray}\label{eq:algalgfl5}
  b_k  & = & \sqrt{\p_{k-1}-\p_k} \nonumber \\
c_k   & = & \sqrt{\q_{k-1}-\q_k}.
 \end{eqnarray}
Following closely the arguments of Section \ref{sec:thirdlev}, we compute the most inner integral and find
\begin{eqnarray}\label{eq:algalgnegprac24a0}
f_{(z)}^{(r)} & = & \mE_{{\mathcal U}_2} e^{\c_2| \sum_{k=2}^{r+1} c_k \h_1^{(k)}  |}
  \nonumber \\
 & = &
  \frac{1}{2}
 e^{\frac{(1-\q_2)\c_2^2}{2}}
 \Bigg(\Bigg.
 e^{-\c_2\zeta_r  }
 \erfc\lp - \lp\c_2\sqrt{1-\q_2}-\frac{\zeta_r } {\sqrt{1-\q_2}}\rp\frac{1}{\sqrt{2}}\rp
 \nonumber \\
& &
 + e^{\c_2\zeta_r   }
   \erfc\lp - \lp\c_2\sqrt{1-\q_2}+\frac{\zeta_r } {\sqrt{1-\q_2}}\rp\frac{1}{\sqrt{2}}\rp
   \Bigg.\Bigg)   ,
\end{eqnarray}
where
\begin{eqnarray} \label{eq:algalgnegprac24a0a0}
\zeta_r = \sum_{k=3}^{r+1} c_k \h_1^{(k)}.
\end{eqnarray}
Consequently
\begin{equation}\label{eq:algalgnegprac24a1}
\mE_{{\mathcal U}_{r+1}}\log \lp \dots \mE_{{\mathcal U}_4}  \lp  \mE_{{\mathcal U}_3}  \lp \mE_{{\mathcal U}_2} e^{\c_2|  \sum_{k=2}^{r+1} c_k\h_1^{(k)}    |}\rp^{\frac{\c_3}{\c_2}}  \rp^{\frac{\c_4}{\c_3}} \dots \rp
=  \mE_{{\mathcal U}_{r+1}}\log  \lp  \dots \mE_{{\mathcal U}_4}  \lp  \mE_{{\mathcal U}_3}   \lp f_{(z)}^{(r)}\rp^{\frac{\c_3}{\c_2}} \rp^{\frac{\c_4}{\c_3}} \dots \rp.
    \end{equation}
In a similar fashion we also find
\begin{eqnarray}\label{eq:algalgnegprac24a2}
\eta_r & = & \sum_{k=3}^{r+1} b_k \u_1^{(2,k)}    \nonumber \\
\bar{h}_3 & = &  -\frac{\sum_{k=3}^{r+1} b_k \u_1^{(2,k)} +\kappa}{\sqrt{1-\p_2}} = -\frac{\eta_r+\kappa}{\sqrt{1-\p_2}}    \nonumber \\
\bar{B} & = & \frac{\c_2}{4\gamma_{sq}} 
\nonumber \\
\bar{C}_r & = & \sum_{k=3}^{r+1} b_k \u_1^{(2,k)}  +\kappa   = \eta_r+\kappa \nonumber \\
f_{(zd)}^{(r,f)}& = & \frac{e^{-\frac{\bar{B}\bar{C}_r^2}{2(1-\p_2)\bar{B} + 1}}}{2\sqrt{2(1-\p_2)\bar{B} + 1}}
\erfc\lp\frac{\bar{h}_r}{\sqrt{4(1-\p_2)\bar{B} + 2}}\rp
\nonumber \\
f_{(zu)}^{(r,f)}& = & \frac{1}{2}\erfc\lp-\frac{\bar{h}_r}{\sqrt{2}}\rp,  
   \end{eqnarray}
and
\begin{multline} \label{eq:algalgnegprac24a3}
   \mE_{{\mathcal U}_{r+1}} \log  \lp \dots \mE_{{\mathcal U}_4}  \lp \mE_{{\mathcal U}_3} \lp \mE_{{\mathcal U}_2} e^{-\c_2\frac{\max(\sum_{k=2}^{r+1} b_k \u_1^{(2,k)} +  \kappa  ,0)^2}{4\gamma_{sq}}}\rp^{\frac{\c_3}{\c_2}} \rp^{\frac{\c_4}{\c_3}} \dots \rp =
   \\
=   \mE_{{\mathcal U}_{r+1}} \log   \lp \dots  \mE_{{\mathcal U}_4}  \lp \mE_{{\mathcal U}_3}  \lp f_{(zd)}^{(r,f)}+f_{(zu)}^{(r,f)}\rp^{\frac{\c_3}{\c_2}} \rp^{\frac{\c_4}{\c_3}} \dots \rp .
    \end{multline}
As in Section \ref{sec:thirdlev} one finds $\c_2\rightarrow \infty$, $\gamma_{sq}\rightarrow 0$,
$\q_k\c_2^2\rightarrow \q_k^{(s)}$, $\frac{\c_k}{\c_2}\rightarrow \c_k^{(s)} $, $2\leq k\leq r$, and
\begin{eqnarray}\label{eq:algalgnegprac24a4}
f_{(z)}^{(r)}
&  \rightarrow &
 e^{\frac{\c_2^2-\q_2^{(s)}}{2}}
 \Bigg(\Bigg.
 e^{- \lp \sum_{k=3}^{r+1} c_k \h_1^{(k)}      \rp }
    + e^{\sum_{k=3}^{r+1} c_k \h_1^{(k)}    }
    \Bigg.\Bigg)
\nonumber \\
&  =  &
   e^{\frac{\c_2^2-\q_2^{(s)}}{2}}
 \Bigg(\Bigg.
 e^{- \zeta_r^{(s)} }
    + e^{\zeta_r^{(s)}   }
    \Bigg.\Bigg)
      =
  e^{\frac{\c_2^2-\q_2^{(s)}}{2}}
 \Bigg(\Bigg.
  2\cosh \lp \zeta_r^{(s)} \rp
    \Bigg.\Bigg)  ,
\end{eqnarray}
where
\begin{eqnarray}\label{eq:algalgnegprac24a4a0}
\zeta_r^{(s)} & = & \sum_{k=3}^{r+1} c_k^{(s)} \h_1^{(4)} \nonumber \\
c_k^{(s)} & =  &  \sqrt{\q_{k-1}^{(s)}-\q_k^{(s)}},3\leq k\leq r+1.
\end{eqnarray}
One can then also transform (\ref{eq:algalgnegprac24a2})  into
\begin{eqnarray}\label{eq:algalgnegprac24a5}
\bar{h}_r & = &  -\frac{\eta_r+\kappa}{\sqrt{1-\p_2}}    \nonumber \\
\bar{B} & = & \frac{\c_2}{4\gamma_{sq}} \rightarrow \infty
\nonumber \\
\bar{C}_r & = & \eta_3+\kappa \nonumber \\
f_{(zd)}^{(r,f)}& = & \frac{e^{-\frac{\bar{B}\bar{C}_r^2}{2(1-\p_2)\bar{B} + 1}}}{2\sqrt{2(1-\p_2)\bar{B} + 1}}
\erfc\lp\frac{\bar{h}_r}{\sqrt{4(1-\p_2)\bar{B} + 2}}\rp
\rightarrow 0
\nonumber \\
f_{(zu)}^{(r,f)}& = & \frac{1}{2}\erfc\lp-\frac{\bar{h}_r}{\sqrt{2}}\rp
\rightarrow \frac{1}{2}\erfc\lp \frac{\eta_r+ \kappa}{\sqrt{2}\sqrt{1-\p_2}} \rp.
   \end{eqnarray}
A combination of (\ref{eq:algalgnegprac24}), (\ref{eq:algalgnegprac24a1}), (\ref{eq:algalgnegprac24a3}), (\ref{eq:algalgnegprac24a4}), and (\ref{eq:algalgnegprac24a5}) gives
\begin{eqnarray}\label{eq:algalgnegprac24a6}
-f_{sq}^{(r,f)}(\infty) & = &     \bar{\psi}_{rd}(\p,\q,\c,\gamma_{sq})
\nonumber \\
 & = &
     \frac{1}{2}
\sum_{k=2}^{r+1}(\p_{k-1}\q_{k-1}-\p_k\q_k)\c_k
 \nonumber \\
& &
  - \frac{1}{\c_r}\mE_{{\mathcal U}_{r+1}}\log \lp \dots \mE_{{\mathcal U}_4}  \lp  \mE_{{\mathcal U}_3}  \lp \mE_{{\mathcal U}_2} e^{\c_2|  \sum_{k=2}^{r+1} c_k\h_1^{(k)}    |}\rp^{\frac{\c_3}{\c_2}}  \rp^{\frac{\c_4}{\c_3}} \dots \rp    \nonumber \\
& &   + \gamma_{sq}
 -\alpha\frac{1}{\c_r}\mE_{{\mathcal U}_{r+1}} \log \lp \dots \mE_{{\mathcal U}_4}  \lp \mE_{{\mathcal U}_3}   \lp \mE_{{\mathcal U}_2} e^{-\c_2\frac{\max(\sum_{k=2}^{r+1} b_k \u_1^{(2,k)} +\kappa ,0)^2}{4\gamma_{sq}}}\rp^{\frac{\c_3}{\c_2}} \rp^{\frac{\c_4}{\c_3}} \dots \rp ,
 \nonumber \\
 & = &
       \frac{1}{2}
\sum_{k=2}^{r+1}(\p_{k-1}\q_{k-1}-\p_k\q_k)\c_k
  - \frac{1}{\c_r}
  \mE_{{\mathcal U}_{r+1}}\log  \lp  \dots \mE_{{\mathcal U}_4}  \lp  \mE_{{\mathcal U}_3}   \lp f_{(z)}^{(r)}\rp^{\frac{\c_3}{\c_2}} \rp^{\frac{\c_4}{\c_3}} \dots \rp
   \nonumber \\
& &   + \gamma_{sq}
 -\alpha\frac{1}{\c_r}
  \mE_{{\mathcal U}_{r+1}} \log   \lp \dots  \mE_{{\mathcal U}_4}  \lp \mE_{{\mathcal U}_3}  \lp f_{(zd)}^{(r,f)}+f_{(zu)}^{(r,f)}\rp^{\frac{\c_3}{\c_2}} \rp^{\frac{\c_4}{\c_3}} \dots \rp
  \nonumber
  \\
& \rightarrow &
       \frac{1}{2\c_2}
\sum_{k=2}^{r+1} (\p_{k-1}\q_{k-1}^{(s)}-\p_k\q_k^{(s)})\c_k^{(s)}
\nonumber \\
& &
  - \frac{1}{\c_2\c_r^{(s)}}\mE_{{\mathcal U}_{r+1}}\log  \lp \dots \mE_{{\mathcal U}_4}   \lp \mE_{{\mathcal U}_3}   \lp
  2\cosh \lp \zeta_r^{(s)} \rp
     \rp^{\c_3^{(s)}}  \rp^{\frac{\c_4^{(s)}}{\c_3^{(s)}}} \dots \rp
  \nonumber \\
& &
 -\alpha\frac{1}{\c_2\c_r^{(s)}}  \mE_{{\mathcal U}_{r+1}} \log  \lp \dots \mE_{{\mathcal U}_4}  \lp \mE_{{\mathcal U}_3}  \lp   \frac{1}{2}\erfc\lp \frac{\eta_3+\kappa}{\sqrt{2}\sqrt{1-\p_2}} \rp  \rp^{\c_3^{(s)}} \rp^{\frac{\c_4^{(s)}}{\c_3^{(s)}}} \dots \rp .
    \end{eqnarray}
One then defines a generic equivalent to (\ref{eq:algnegprac24a6a0})
\begin{eqnarray}\label{eq:algalgnegprac24a6a0}
   \bar{\psi}_{rd}^{(r,s)}(\p,\q,\c,\gamma_{sq})
\hspace{-.0in}  & \triangleq & \hspace{-.0in}
       \frac{1}{2}
\sum_{k=2}^{r+1} (\p_{k-1}\q_{k-1}^{(s)}-\p_k\q_k^{(s)})\c_k^{(s)}
\nonumber \\
& &
  - \frac{1}{\c_r^{(s)}}\mE_{{\mathcal U}_{r+1}}\log  \lp \dots \mE_{{\mathcal U}_4}   \lp \mE_{{\mathcal U}_3}   \lp
  2\cosh \lp \zeta_r^{(s)} \rp
     \rp^{\c_3^{(s)}}  \rp^{\frac{\c_4^{(s)}}{\c_3^{(s)}}} \dots \rp
  \nonumber \\
& &
 -\alpha\frac{1}{\c_r^{(s)}}  \mE_{{\mathcal U}_{r+1}} \log  \lp \dots \mE_{{\mathcal U}_4}  \lp \mE_{{\mathcal U}_3}  \lp   \frac{1}{2}\erfc\lp \frac{\eta_r+\kappa}{\sqrt{2}\sqrt{1-\p_2}} \rp  \rp^{\c_3^{(s)}} \rp^{\frac{\c_4^{(s)}}{\c_3^{(s)}}} \dots \rp  ,
    \end{eqnarray}
 and recognizes that condition $f_{sq}(\infty)=0$ can be replaced by $  \bar{\psi}_{rd}^{(r,s)}(\p,\q,\c,\gamma_{sq}) = 0$ for any $r\geq 2$. After numerical evaluations we find
\begin{equation}\label{eq:algnegprac25a0}
\hspace{-1in}(\mbox{\textbf{fourh and fifth level:}}) \qquad \qquad  \alpha_c^{(4)}(0) \approx
 \red{\mathbf{0.7777}} \qquad \mbox{and} \qquad  \alpha_c^{(5)}(0) \approx
 \red{\mathbf{0.7764}}.
  \end{equation}
All relevant parametric values for the first four levels of lifting (1,2,3,4-sfl RDT) are systematically shown in Table \ref{tab:tab3}.
\begin{table}[h]
\caption{$r$-sfl RDT parameters  ($r\leq 4$);   $\hat{\c}_2\rightarrow \infty$;   $\hat{\c}_k^{(s)} = \lim_{\hat{\c}_2\rightarrow\infty}\frac{\hat{\c}_k}{\hat{\c}_2},\hat{\q}_k^{(s)} =\lim_{\hat{\c}_2\rightarrow\infty} \hat{\q}_k\hat{\c}_2^2,k\geq 2$; $\kappa=0$; $n,\beta\rightarrow\infty$}\vspace{.1in}
\centering
\def\arraystretch{1.2}
{\small \begin{tabular}{||l||c||c|c|c|c||c|c|c|c||c|c|c||c||}\hline\hline
 \hspace{-0in}$r$                                              & $\hat{\gamma}_{sq}$    &  $\hat{\p}_4$   &  $\hat{\p}_3$  &  $\hat{\p}_2$ & $\hat{\p}_1$     & $\hat{\q}_4^{(s)} $  & $\hat{\q}_3^{(s)} $  & $\hat{\q}_2^{(s)} $  & $\hat{\q}_1$ &  $\hat{\c}_4^{(s)}$  &  $\hat{\c}_3^{(s)}$  &  $\hat{\c}_2$    & $\alpha_c^{(r)}(0)$  \\ \hline\hline
$1$                                       & $0.399$  & $0$  &  $0$  &  $0$   & $\rightarrow 1$  &  $0$ &  $0$    & $0$ & $\rightarrow 1$
& $ \rightarrow 0 $  & $ \rightarrow 0 $ &  $\rightarrow 0$  & \bl{$\mathbf{1.2732}$} \\ \hline
   $2$                                      & $0$ & $0$ & $0$   & $0.564$ & $\rightarrow 1$   &  $0$  &   $0$  & $2.576$ & $\rightarrow 1$
& $ \rightarrow 0 $ & $ \rightarrow 0 $ &  $\rightarrow \infty$   & \bl{$\mathbf{0.8331}$}  \\ \hline\hline
   $3$                                       & $0$ & $0$ & $0.648$    & $0.984$ & $\rightarrow 1$  &  $0$  &  $0.248$   &  $1.021$ & $\rightarrow 1$
& $ \rightarrow 0 $ & $4.33$  &  $\rightarrow \infty$   & \red{$\mathbf{0.7843}$}  \\ \hline
   $4$                                      & $0$ & $0.705$ & $0.942$    & $0.991$ & $\rightarrow 1$ &  $0.280$  &  $0.488$   &  $1.860 $ & $\rightarrow 1$
& $ 2.407 $  & $5.390$  &  $\rightarrow \infty$   & \red{$\mathbf{0.7777}$}  \\ \hline \hline
  \end{tabular}}
\label{tab:tab3}
\end{table}
Corresponding values for the first five levels (1,2,3,4,5-sfl RDT) are in Table \ref{tab:tab4}.
\begin{table}[h]
\caption{$r$-sfl RDT parameters  ($r\leq 5$);   $\hat{\c}_2\rightarrow \infty$;   $\hat{\c}_k^{(s)} = \lim_{\hat{\c}_2\rightarrow\infty}\frac{\hat{\c}_k}{\hat{\c}_2},\hat{\q}_k^{(s)} =\lim_{\hat{\c}_2\rightarrow\infty} \hat{\q}_k\hat{\c}_2^2,k\geq 2$; $\kappa=0$; $n,\beta\rightarrow\infty$; $\hat{\p}_1\rightarrow 1$, $\hat{\q}_1\rightarrow 1$. }\vspace{.1in}
\centering
\def\arraystretch{1.2}
{\small \begin{tabular}{||l||c||c|c|c|c||c|c|c|c||c|c|c||c||}\hline\hline
 \hspace{-0in}$r$                                              & $\hat{\gamma}_{sq}$    &  $\hat{\p}_5$   &  $\hat{\p}_4$   &  $\hat{\p}_3$  &  $\hat{\p}_2$    & $\hat{\q}_5^{(s)} $   & $\hat{\q}_4^{(s)} $  & $\hat{\q}_3^{(s)} $  & $\hat{\q}_2^{(s)} $  &  $\hat{\c}_5^{(s)}$  &  $\hat{\c}_4^{(s)}$  &  $\hat{\c}_3^{(s)}$     & $\alpha_c^{(r)}(0)$  \\ \hline\hline
$1$                                       & $\hspace{-.02in}0.399\hspace{-.02in}$  & $0$  & $0$  &  $0$  &  $0$     &  $0$  &  $0$ &  $0$    & $0$
& $ \rightarrow 0 $  & $ \rightarrow 0 $ &  $\rightarrow 0$  & \bl{$\mathbf{1.2732}$} \\ \hline
   $2$                                      & $0$ & $0$ & $0$ & $0$   & $0.564$  & $0$  &  $0$  &   $0$  & $2.576$
& $ \rightarrow 0 $ & $ \rightarrow 0 $ &  $\rightarrow 0$   & \bl{$\mathbf{0.8331}$}  \\ \hline\hline
   $3$                                       & $0$ & $0$ & $0$ & $0.648$    & $0.984$  & $0$ &  $0$  &  $0.248$   &  $1.021$
 &  $\rightarrow 0$  & $ \rightarrow 0 $ & $4.33$   & \red{$\mathbf{0.7843}$}  \\ \hline
   $4$                                      & $0$ & $0$ & $0.705$ & $0.942$    & $0.991$   & $0$ &  $0.280$  &  $0.488$   &  $1.860 $
&  $\rightarrow 0$  & $ 2.407 $  & $5.390$    & \red{$\mathbf{0.7777}$}  \\ \hline
   $5$                                      & $0$ & $\hspace{-.02in}0.685\hspace{-.02in}$ & $0.946$ & $0.976$    & $0.996$    & $\hspace{-.02in}0.209\hspace{-.02in}$   &  $0.383$  &  $0.739$   &  $3.017 $
&  $  1.662   $   & $  3.483  $  & $5.716$   & \red{$\mathbf{0.7764}$}  \\ \hline\hline
  \end{tabular}}
\label{tab:tab4}
\end{table}

\subsubsection{Algorithmic implications discussion}
\label{sec:algimp}

Looking carefully at Tables \ref{tab:tab3} and  \ref{tab:tab4}, one now observes that the lifting mechanism indeed starts to converge as $r$ gets larger. Moreover, all indications are that the converging value is precisely in the above discussed range $0.77-0.78$ (quite likely somewhere in $0.775-0.776$ interval). This is in a remarkable agrement with results of \cite{Bald16} and \cite{Stojnicabple25} where the local entropy considerations estimate that atypical clustering defragmentation also happens in $0.77-0.78$ range (numerical requirements in both  \cite{Bald16} and \cite{Stojnicabple25} are rather heavy which makes further narrowing down of the predicted range a bit more challenging). Nonetheless the agreement seems rather magical. While it may be difficult right now to see if there is indeed any connection between the concepts that we discuss here and those from  \cite{Bald16,Stojnicabple25}, the closeness of the obtained numerical results suggests that some intrinsic connection  likely exists. If it indeed exists then it is also likely a consequence of a generic principle as we have not used here any properties beyond core  fl RDT parametrization. Studying further connections with overlap gap properties (OGP) and continuity of associated Gibbs measures might provide more concrete answers as to what the role of the obtained thresholds within the algorithmic ABP context is. One is in first place interesting whether or not ABP indeed exhibits statistical-computational gap and if it does, how far away the algorithmic threshold is from the estimates obtained here.

It is also interesting to draw a parallel with recent algorithmic results obtained in \cite{Stojniccluphop25} where a particular variant of the  CLuP (controlled loosening-up) algorithm is employed to determine GSEs of positive and negative Hopfield models (Hop+ and Hop-). Excellent algorithmic performance  that very closely matches theoretical GSE predictions is obtained in both scenarios. In particular, for Hop+ one expects absence of statistical-computational  gap and the algorithmic results from  \cite{Stojniccluphop25} indeed indicate that such expectation is likely correct. Similar proximity of practical performance and theoretical predictions is observed for the negative variant as well. However, differently from the positive model, the corresponding negative model theoretical predictions assume removal of the decreasing ordering of $\c$ sequence which is unphysical and in a way resembles what happens here in the ABP context. In other words, for Hop- there are algorithms that closely approach the same type of theoretical prediction which suggests that one might expect existence of similar algorithms here as well. Also, again differently from the positive variant, the Hop- exhibits the discontinuity of the associated Gibbs measures (with or without $\c$ sequence ordering) and the statistical-computational gap is expected. Along those lines, the higher lifting level results obtained in \cite{Stojniccluphop25} are either precisely the algorithmic thresholds or fairly close to it. Moreover, in such a constellation it is also likely that the 2-sfl RDT results (the highest RDT lifting level where decreasing property of $\c$ sequence still holds) are the theoretical GSE values (reachable with an infinite computational power). Numerical difference between the second and higher level estimates is insignificant in Hop- and for all practical purposes finding its GSE is basically easy. In other words, predicated existence of statistical-computational gap is more a mere formality than a practically relevant feature. On the other hand, here in the ABP context, the difference between the second and higher levels is substantial and the predicated statistical-computational gap is a more relevant obstacle that has to be kept in mind. Nonetheless, overall similarity between underlying phenomenologies of Hop- and ABP seems rather fascinating and points to a likely presence of a universal principle that connects them.

Finally, we also note the following property of the above machinery. After repeating the above calculations on all five levels relying on modulo-$\m$ concepts from  \cite{Stojnicflrdt23} we obtained exactly the same results as those in Tables \ref{tab:tab1}--\ref{tab:tab4}. This effectively uncovers that the $\c$ \emph{stationarity} is of \emph{maximization} type which remarkably matches the very same behavior observed in \cite{Stojnicbinperflrdt23,Stojnicnegsphflrdt23}. While the prior results may have suggested such a behavior on the first two levels, the fact that it extends to higher levels (where the meaning of the evaluated quantities deviates from the satisfiability threshold) is very intriguing and quite likely again consequence of a more  general underlying  principle.

\section{Conclusion}
\label{sec:conc}

The paper studies the capacity of the classical asymmetric binary perceptron (ABP). Recent theoretical and algorithmic studies suggest that ABP likely exhibits the so-called statistical-computational gap. This implies existence of two phase transitions in its statistical behavior: \textbf{\emph{(i)}} there is a critical constraint density, \emph{satisfiability threshold} $\alpha_c$, below/above which ABP succeeds/fails to operate as a storage memory; and \textbf{\emph{(ii)}} there is a critical constraint density, \emph{algorithmic threshold} $\alpha_a$, below/above which one can/cannot efficiently determine ABP's weights so that it operates as a storage memory.

We here focused on a particular parametric utilization of a powerful mathematical engine called \emph{fully lifted random duality theory} (fl RDT) \cite{Stojnicflrdt23} and studied its potential algorithmic implications.  This allowed us to uncover that fl RDT exhibits a structural parametric change as one progresses through lifting levels. On the first two levels,  the so-called $\c$ sequence, one of the key parametric fl RDT components, is of the natural (physical) decreasing type. As one moves to higher levels this phenomenology changes and a perfect $\c$ ordering is not present any more. Such a change is then connected to the change from satisfiability to algorithmic threshold. Namely, concrete numerical value that we obtain for constraint density $\alpha$ on the second lifting level precisely matches the satisfiability threshold $\alpha_c\approx 0.8331$. As one progresses through higher lifting levels the estimate of the critical constraint density decreases. However, already on the fifth lifting level we obtain $\alpha\approx 0.7764$ with a clear converging tendency and the difference between successive lifting levels estimates on the order of $\sim 0.001$.

The above developments are then observed to be in an excellent agreement with recent studies of ABP's atypical solutions clusterings. In particular, they are shown to fairly closely match the so-called local entropy results of \cite{Bald16,Stojnicabple25} which predict the clustering defragmentation (likely responsible for failure of locally improving algorithms)  for $\alpha\in (0.77,0.78)$. Drawing further parallel with recent algorithmic studies of the negative Hopfield models (Hop-),  a structural similarity in parametric behavior of the above mentioned $\c$ sequence in two models is observed. Given that there are efficient algorithms that can approach higher lifting levels theoretical  Hop- predictions, it might not come as a surprise that equally successful conceptually similar ones can be designed for ABP as well. Along the same lines, the presented parametric algorithmic phenomenology seems to be a generic fl RDT feature and exploring further to what extent it applies to other well known models would be very interesting.

\begin{singlespace}
\bibliographystyle{plain}
\bibliography{nflgscompyxRefs}
\end{singlespace}

\end{document}